%% file: main.tex
\documentclass[10pt,journal,compsoc]{IEEEtran}

\usepackage{amsmath,amsfonts,amsthm,amssymb,amstext}
\usepackage[nocompress]{cite}
\usepackage[caption=false,font=normalsize,labelfont=sf,textfont=sf]{subfig}
\usepackage{algorithm}
\usepackage[noend]{algpseudocode}
\usepackage{hyperref}
\usepackage{booktabs}
\usepackage{microtype}
\usepackage{float}
\usepackage{xcolor}
\usepackage{multirow}
\usepackage{bbding}
\usepackage{graphicx}
\usepackage{url}

\newcommand{\w}{\mathbf{w}}
\newcommand{\C}{\mathbf{C}}

\newcommand{\e}{\mathbf{e}}
\newcommand{\OO}{\mathbf{O}}
\newcommand{\x}{\mathbf{x}}

\newtheorem{Theorem}{Theorem}
\newtheorem{Lemma}{Lemma}
\newtheorem{Remark}{Remark}
\newtheorem{Proposition}{Proposition}
\newtheorem*{Definition}{Definition}

\begin{document}
\title{Diffusion Mechanism in Residual Neural Network: Theory and Applications}

\author{Tangjun~Wang, Zehao~Dou, Chenglong~Bao,
        and~Zuoqiang~Shi
\IEEEcompsocitemizethanks{
    \IEEEcompsocthanksitem T.Wang is with the Department of Mathematical Sciences, Tsinghua University, Beijing 100084, China. 
    \IEEEcompsocthanksitem Z. Dou is with the Department of Statistics and Data Science, Yale University. Work done during his undergraduate study at Peking University.
    \IEEEcompsocthanksitem C. Bao is with Yau Mathematical Sciences Center, Tsinghua University, Beijing 100084, China, and Yanqi Lake Beijing Institute of Mathematical Sciences and Applications, Beijing 101408, China. E-mail: clbao@mail.tsinghua.edu.cn
    \IEEEcompsocthanksitem Z. Shi is with Yau Mathematical Sciences Center, Tsinghua University, Beijing 100084, China, and Yanqi Lake Beijing Institute of Mathematical Sciences and Applications, Beijing 101408, China. E-mail: zqshi@tsinghua.edu.cn
\IEEEcompsocthanksitem Corresponding authors: C. Bao and Z. Shi.
}
}

\IEEEtitleabstractindextext{
	\begin{abstract}
		\input{data/abstract}
	\end{abstract}
	
}

\maketitle

\IEEEraisesectionheading{\section{Introduction}}

\input{data/introduction}

\section{Related Works}
\input{data/related_works}

\section{Diffusion Residual Networks}
\input{data/diffresnet}

\section{Analysis of Diffusion Mechanism}
\input{data/theory}

\section{Experiments}
\input{data/synthetic}
\input{data/graph}
\input{data/fewshot}

\section{Conclusion}
\input{data/conclusion}

\bibliographystyle{IEEEtran}
\bibliography{reference.bib}

\input{data/biography}

\appendices
\input{data/appendix_stability}
\input{data/appendix_proof1}
\input{data/appendix_proof2}
\input{data/appendix_proof3}
\input{data/appendix_experiments}

\end{document}

%% file: data/abstract.tex
Diffusion, a fundamental internal mechanism emerging in many physical processes, describes the interaction among different objects. 
In many learning tasks with limited training samples, the diffusion connects the labeled and unlabeled data points and is a critical component for achieving high classification accuracy. 
Many existing deep learning approaches directly impose the fusion loss when training neural networks.
In this work, inspired by the convection-diffusion ordinary differential equations~(ODEs), we propose a novel diffusion residual network (Diff-ResNet), internally introduces diffusion into the architectures of neural networks. 
Under the structured data assumption, it is proved that the proposed diffusion block can increase the distance-diameter ratio that improves the separability of inter-class points and reduces the distance among local intra-class points. Moreover, this property can be easily adopted by the residual networks for constructing the separable hyperplanes.
Extensive experiments of synthetic binary classification, semi-supervised graph node classification and few-shot image classification in various datasets validate the effectiveness of the proposed method.



%% file: data/introduction.tex
\IEEEPARstart{R}{esNet} \cite{he2016deep} and its variants, containing skip connections among different layers, are promising network architectures in deep learning.
Compared to non-residual networks, ResNet significantly improve the training stability and the generalization accuracy. To understand the success of ResNet, 
a recent line of works build up its connection with ordinary differential equations (ODEs)~\cite{Weinan2017A,chen2018neural,haber2018stable}.  Let $x\in\mathbb{R}^n$ be a data point, the ODE model of a ResNet is
\begin{equation}
\label{eq:introduction}
\frac{dx(t)}{dt}=f(x(t),\theta(t)), \quad x(0) = x.
\end{equation}
where $f(x,\theta)$ is a map parametrized by $\theta$. It is straightforward that the forward Euler discretization of \eqref{eq:introduction} recovers the residual connection, which motivates the connections between ResNets and ODE. Based on the above observation, many recent works are proposed from two perspectives: the ODE inspired neural networks and the neural network based ODE. In concrete, the attempts for the ODE inspired neural networks can be classified into two directions. One approach for designing networks is to unroll the ODE system via different discretization schemes, 
which build up an end-to-end mapping. Typical networks include PolyNet~\cite{zhang2017polynet}, FractalNet~\cite{larsson2016fractalnet} and linear multi-step network~\cite{lu2018beyond}. 
The other approach is to add some new blocks into the current network architecture by the modification of the ODE model, e.g. noise injection~\cite{Gastaldi17ShakeShake}, 
stochastic dynamic system~\cite{huang2016deep}, adding a damping term~\cite{yang2020interpolation}. Due to the strong mathematical foundation of ODE, the network architectures proposed in the above works have shown the improved explainability and performance. On the other hand, the neural network based ODE model parametrizes the velocity $f$ by a neural network and finds the parameters $\theta$ via the optimal control formulation~\cite{lagaris1998artificial,dissanayake1994neural}. These methods improves the expressive ability of a traditional ODE method and exhibit promising results in various problems, including systems with irregular boundaries~\cite{lagaris2000neural,mcfall2009artificial},  PDEs in the field of fluid mechanics~\cite{baymani2010artificial}, and high-dimensional differential equations~\cite{han2018solving}. Thus,  connection between ResNets and ODE deserves deep exploration.

Success of deep learning methods highly depends on a large amount of training samples, but collecting training data requires intensive labor works, and is sometimes impossible in many application fields due to the privacy or safety issues. To alleviate the dependency of training data, semi-supervised learning (SSL)~\cite{zhu2009introduction,chapelle2009semi} and few-shot learning (FSL)~\cite{fei2006one, vinyals2016matching} have received great interests in recent years. Semi-supervised learning typically uses a large amount of unlabeled data, together with the labeled data, to construct better classifiers. Few-shot learning is a more recent paradigm which is closely related to semi-supervised learning, and the main difference lies in that the the size of support set (labeled points) is much smaller. One common feature in SSL and FSL is to make use of the unlabeled samples to address the limited labeled set issue. 
See \cite{chapelle2009semi, wang2020generalizing} for the review of SSL and FSL. In this work, we focus on the deep learning based approaches for solving SSL problems. In general, the deep SSL methods can fall into two categories: consistency regularization and entropy minimization~\cite{oliver2018realistic}. Consistency regularization demands that minor perturbation on the input does not change the output significantly. $\Pi$-Model~\cite{bachman2014learning, laine2016temporal} and its more stable version Mean Teacher~\cite{tarvainen2017mean} are based on this idea, which require the stochastic network predictions over different passes 
have little disturbance. VAT~\cite{miyato2018virtual} replaces stochastic perturbation with the "worst" perturbation which can most significantly affect the output of the prediction function. Entropy minimization, which is closely related to self-training, encourages more confident predictions on unlabeled data. EntMin~\cite{grandvalet2004semi} impose the low entropy requirements on the predictions of unlabeled examples. Pseudo label~\cite{lee2013pseudo} feeds unlabeled samples with high prediction confidence into the network as labeled ones to train better classifier. Besides, some holistic approaches try to unify the current effective methods in SSL in a single framework, e.g.\ MixMatch~\cite{berthelot2019mixmatch}, FixMatch~\cite{sohn2020fixmatch}. Despite the existence of many deep SSL methods that achieve impressive results in various tasks, the internal mechanism of the consistency regularization or entropy minimization methods remains unclear in SSL/FSL classification. 

To demystify this mystery in SSL and FSL, we propose an ODE inspired deep neural network that is based on the connection between ODE and ResNet. As shown in~\eqref{eq:introduction}, current ODE counterpart of ResNet is a convection equation. Each point governing by~\eqref{eq:introduction} is evolved independently. This evolution process is acceptable when a large amount of training samples are available, but the performance is significantly deteriorated as the number of supervised samples decreases. Thus, it may be problematic when directly applying~\eqref{eq:introduction} for SSL/FSL. To solve this problem, we introduce diffusion mechanism in~\eqref{eq:introduction}, leading to a convection-diffusion equation. After the discretization, we obtain a diffusion based residual neural network. The imposed diffusion to enforce the interactions among samples (include labled and unlabled) that is a key component in the regime of limited training data. In fact, it is worth mentioning that the convection and diffusion mechanisms always appear simultaneously in complex systems such as fluid dynamics~\cite{morton2019numerical}, building physics~\cite{svoboda2000convective}, semiconductors~\cite{markowich1989system}, which strongly motivates the integration of diffusion into deep ResNets. 

Imposing the interactions among samples is a classical idea and has appeared in many existing SSL approaches~\cite{zhu2003semi,nadler2009semi,shi2017weighted}, but the combination of convection and diffusion in the network architecture is underexplored. In addition, most methods introduce the diffusion by adding a Laplacian regularizer in the loss function, which is widely used in graph-based SSL~\cite{weston2012deep, yang2016revisiting}. In this case,  tuning the weight of the Laplacian regularizer is not an easy task and often sensitive to tasks. Different from the above methods, we explicitly add diffusion layers into the ResNet. The proposed diffusion layers internally impose the interactions among samples and have shown to be more effective in SSL/FSL. More importantly, we theoretically analyze the diffusive ODE and show its advantage in terms of distance-diameter ratio among data samples, which provides a solid foundation the proposed method. 
In summary, we list our main contributions as follows.
\begin{itemize}
    \item We propose a convection-diffusion ODE model for solving SSL/FSL, leading to the addition of diffusion layers into ResNets after proper discretization. The proposed diffusion based ResNet strengthen the relationships among labeled and unlabeled data points via a designed network architecture, rather than imposing the diffusion loss in the total. To the best of our knowledge, this is the first attempt that internally incorporates diffusion mechanism into deep neural network architecture.
    \item Under the structured data assumption \cite{li2018learning}, it is proved that diffusion mechanism is able to accelerate the classification process in the sense that samples from different subclasses can be driven apart, while samples from the same subclass will be brought together. Using such property, we can theoretically construct a residual network that ensures that output features are linearly separable. This analysis provides the mathematical foundation of our method.
    \item Extensive experiments on various tasks and datasets validate our theoretical results and the advantages of the proposed Diff-ResNet.   
\end{itemize}

The rest of this paper is organized as follows. The related work is given in Section
2. Section 3 presents the formulation and details of our diffusion residual network, and Section 4 provides theoretical the analysis of diffusion mechanism. Experimental results on various tasks are reported in Section 5. We conclude the paper in Section 6.


%% file: data/related_works.tex
\subsection{Diffusion Mechanism}
The idea of diffusion is widely used in various fields. In graph neural networks, \cite{klicpera2019diffusion} concludes a unified framework for graph diffusion, and proposes a preprocessing method that create a new graph based on diffusion. With spectral analysis of the new graph, they show that local clusters can be amplified while noise can be suppressed. Diffusion-Convolutional neural networks~\cite{atwood2016diffusion} learn diffusion-based representations from graph and use them as an effective basis for node classification. Diffusion is also used in diffusion map or eigenmap~\cite{coifman2005geometric,belkin2003laplacian}, which uses linear diffusion PDEs with closed
form solutions for dimension reduction. Different from linear dimensional reduction methods like principal component analysis (PCA), diffusion maps belongs to nonlinear methods that focus on the underlying manifold of data. It constructs a Markov chain based on diffusion process, which can capture the geometric structure of manifold at larger scales as the diffusion goes on. Diffusion is used by previous work to deal with data insufficiency. \cite{kushnir2020diffusion} diffuses the label information to propose an efficient criterion for switching between exploration and refinement in active learning. Recently, diffusion has been proposed to design new network architectures. DifNet~\cite{jiang2018difnet} constructs a diffusion process on a single image for semantic segmentation, and approximates the process by a cascade of random walks. \cite{wang2018resnets} also adds a diffusion term into ODE induced by ResNet, but its diffusion is in the Euclidean space while ours is in the embedded manifold. Graph Neural Diffusion~\cite{chamberlain2021grand} also uses parabolic diffusion-type PDEs to design GNNs, but they introduce additional attention parameters in each diffusion layer. Additionally, their final output is calculated directly from inputs by performing diffusion, while our diffusion layers are added before each residual block. Thus, we can integrate convection and diffusion in an intrinsic way, theoretically guarantee the improvement in classification accuracy. EPNet~\cite{rodriguez2020embedding} uses diffusion on both embedded features and labels to utilize the query set in few-shot learning. However, their method only contains diffusion, while ours combine convection and diffusion internally. To the best of our knowledge, this is the first work that applies the diffusion mechanism to ResNet with rigorous mathematical analysis.



\subsection{Neural ODEs}
The deep learning models and dynamical systems have closed relationship, which is firstly introduced in a proposal of E. et al.~\cite{Weinan2017A}. Using this connection, many works have been proposed for improving deep learning models. \cite{li2017maximum,li2018an} propose several training algorithms based on Pontragyn's Minimum Principle condition and successive approximation method. Neural ODE \cite{chen2018neural} treats ResNet as the forward Euler discretization of an ordinary differential equation and adopt adjoint method to train the ODE model, which inspires a long list of work considering the relationship between ordinary differential equations and deep residual networks. These papers interpret ResNets as a discretization of dynamical systems, where the dynamics at each step is a linear transformation followed by a non-linear activation function. \cite{haber2018stable} treats deep networks as a parameter estimation problem of nonlinear dynamical systems, and propose new forward propagation techniques that relieve exploding or vanishing gradients problem. \cite{lu2018beyond} provides a unified framework for interpreting ResNets and its derivatives, such as PolyNet~\cite{zhang2017polynet} and FractalNet~\cite{larsson2016fractalnet}. Based on the framework, the author proposes a linear multi-step architecture. However, most ODE inspired residual networks cannot be directly applied to the semi-supervised problems as they need many supervised samples.



\subsection{Graph-based semi-supervised learning}
Graph-based SSL algorithms have received much attention~\cite{zhu2003semi,zhu2009introduction} because graph structure can effectively encode the relationship among data points. Graph-based semi-supervised learning is based on the assumption that nearby nodes tend to have the same labels. In graph, each sample is denoted by a vertex, and the weighted edge measures the similarity between samples. \cite{zhu2003semi} initially proposes the Gaussian Fields and Harmonic Functions (GFHF) algorithm, which aims to minimize the graph Laplacian objective function with the constraint on labeled points. After that, \cite{zhou2004learning} introduces Local and Global Consistency (LGC) algorithm, which differs from GFHF model in that the label for each sample is penalized to ensure regularity, and the hard label constraint is turned into a soft constraint using Laplacian multiplier. Belkin et al.\cite{belkin2004semi,belkin2006manifold} proposes the manifold regularization framework, which employs a kernel-based regularization term. Such kernels are often derived from the graph Laplacian, which becomes a general extension of graph Laplacian regularization~\cite{ando2007learning, smola2003kernels}. Semi-supervised embedding~\cite{weston2012deep} extends the Laplacian regularizer from labels to network outputs, which imposes constraints on the parameters of a neural network. Recent explosion in SSL can be traced back to SimCLR~\cite{chen2020simple}, which provides a simple framework for contrastive learning of visual representations that can be used on semi-supervised tasks. SwAV~\cite{caron2020unsupervised} both clusters data and enforces consistency between cluster assignments produced for different views of the same image simultaneously to avoid the computation burden of a large number of explicit pairwise feature comparisons in contrastive learning. DINO~\cite{caron2021emerging} extends the self-supervised methods to transformer~\cite{dosovitskiy2021image} and further improves the performance on SSL benchmarks. These methods are based mostly on contrastive learning and self-supervised learning, and are not directly related to graph-based semi-supervised learning. Few-shot learning, which is a special scenario of semi-supervised learning, may also be integrated with graph by introducing a Gaussian kernel similarity matrix on the embedded features obtained through a pretrained backbone. \cite{satorras2018few} uses label propagation by building a similarity matrix on both support set and query set. wDAE-GNN~\cite{gidaris2019generating} introduces a denoising autoencoder that injects Gaussian noise on a set of classification weights as inputs and learns to reconstruct the target classification weights, in order to regularize the weight. Furthermore, they implement the denoising autoencoder as a graph neural network to capture the co-dependencies. EPNet~\cite{rodriguez2020embedding} uses both embedding propagation on features and label propagation on labels to utilize the information in unlabeled query set. EASE~\cite{zhu2022ease} proposes an assumption that embedded features can be drawn from multiple subspaces, and thus constructs the similarity matrix in a block-diagonal prior. Nonetheless, our paper embeds the Laplacian regularization intrinsically in the neural network structure through diffusion layers, which is different from methods that use iterative approach to minimize the loss function, or those that adds a regularization term based on graph Laplacian to the objective function and uses vanilla networks to optimize.


\begin{figure*}[t!]
	\centering
	\includegraphics[width=0.8\linewidth]{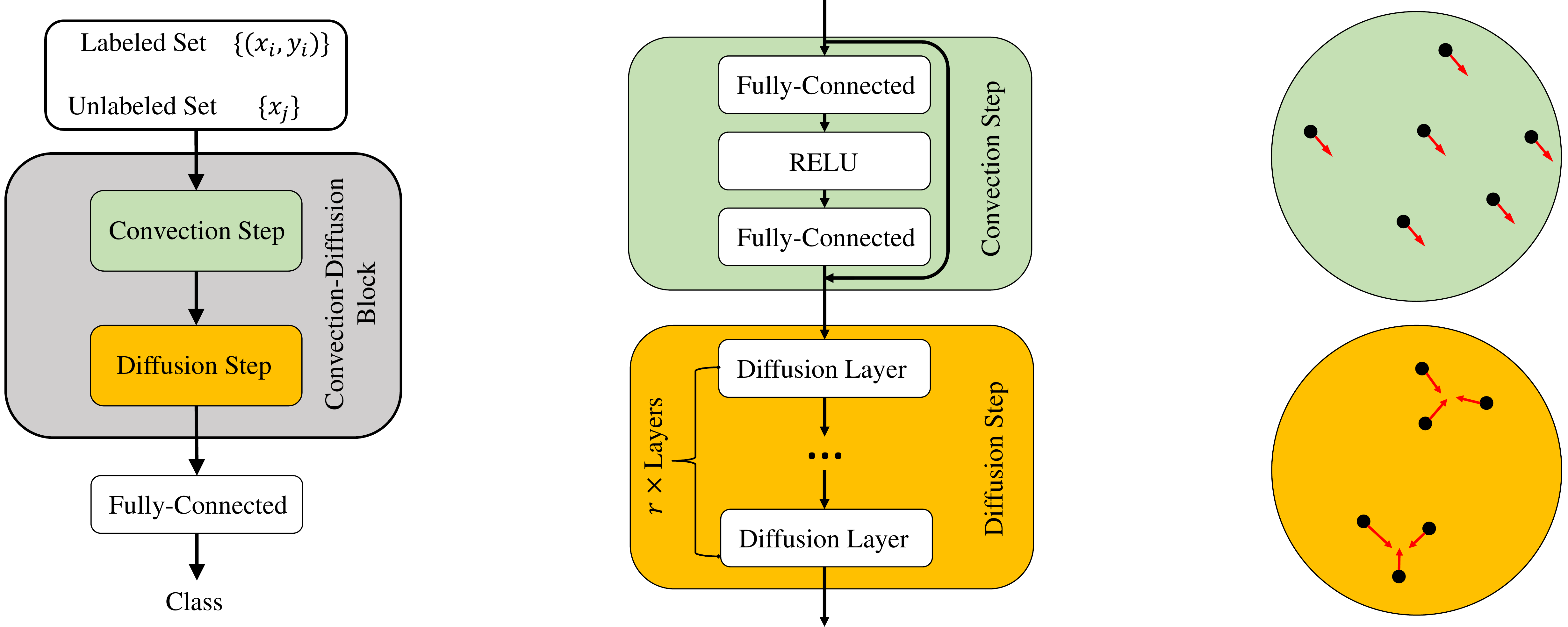}
	\caption{Illustration of our Diff-ResNet. Left: network structure; middle: details in each convection-diffusion block; right: movements of data points with convection or diffusion.}
	\label{fig:network}
\end{figure*}

\subsection{Few-shot learning}
To address the limited training samples problem, few-shot learning, a new learning paradigm~\cite{fei2006one,miller2000learning}, has been proposed and become an important topic in machine learning. The few-shot learning has been extensively explored in recent years, and there are many different kinds of methods. Among existing few-shot learning methods, embedding learning is a typical approach, which maps each sample to a low dimensional spaces such that similar samples are close while dissimilar samples are far away. The embedding method maps data samples to a feature space by training an embedding function from a large-scale dataset. In the feature space, another classifier is applied for classifying query data. Typical methods include Matching Network~\cite{vinyals2016matching}, which uses LSTM~\cite{hochreiter1997long} with attention mechanism and external memory to construct the embedding mappings for the query set and the support set. It firstly introduce the episode-based training method to match the training and testing condition. Prototypical Network~\cite{snell2017prototypical} compares the embedding of the query point with the prototype of each class and assign the query point to the same class of the nearest prototype. The work in \cite{sung2018learning} introduces Relation Network that concatenates features of training and test samples as the embedding, and feeds it to another CNN to output a similarity score. It also introduces a deep distance metric instead of hand-crafted metrics like cosine distance. TADAM~\cite{oreshkin2018tadam} changes the metric for different tasks in order to use the specific information of each task. It introduces new parameters to scale the gradient and fine-tune the output of each convolutional blocks. 
Recently, Wang et al.~\cite{wang2019simpleshot} shows that merely pretraining an embedding function on base classes along with nearest neighbor clustering in $l_2$ distance can achieve competitive results. This line of work avoid complicated training strategies and get back on simple yet effective manipulations on embedded features. Following that, Laplacian regularized clustering~\cite{ziko2020laplacian} adopts a regularizer based on graph Laplacian. \cite{liu2019prototype} rectifies the features to reduce the cross-class and intra-class bias, and use the rectified prototype to help clustering. Several works have been focusing on the distribution of embedded features, e.g. \cite{yang2021free,hu2021leveraging} assumes that the features should be drawn from a Gaussian distribution, and thus uses Tukey’s Transformation Ladder\cite{tukey1977exploratory}, also called Power Transform, to calibrate the distribution. The specific setting of current evaluation of few-shot learning performance, which draws \textit{exactly the same} number of query samples from each class, inspires the application of Sinkhorn algorithm from the field of optimal transport. \cite{huang2019few,hu2021leveraging} proposes Sinkhorn k-means, and \cite{zhu2022ease} extends the algorithm to a semi-supervised setting. However, we do not explicitly use the uniform distribution of query labels because we believe it is impractical in real-world tasks. Compared to the existing works, our method adopts similar embedding training procedure with different attached classifier that has good performance in various few-shot learning tasks. More importantly, under the suitable assumption, we establish a thorough theoretical analysis of the propose method.

%% file: data/diffresnet.tex
In this section, we introduce the diffusion mechanism from the ODE perspective and present the Diff-ResNet based on the numerical scheme for the diffusive ODE.
\subsection{The ODE formulation}
In ResNet~\cite{he2016deep}, the feature map of a specific data point $x_i$ after the $k$-th residual block is defined as $x_i^k$. Residual connection means $x_i^k$ is added to $x_i^{k+1}$ via a skip identity link. If we gather convolutional layers, batch normalization layers and other layers together, and denote them as function $f$, each residual block can be written as 
\begin{equation}
\label{eq:residual_block}
x_i^{k+1} = x_i^k + f(x_i^k,\theta^k) ,
\end{equation}
where $\theta^k$ is the parameters of $k$-th block. From ODE perspective, $f$ can be seen as the velocity, while $x_i^k$ and $x_i^{k+1}$ can be treated as the start position and end position of $x_i$. Introducing a time step $\Delta t$ which can be absorbed in $f$, the ResNet can be seen as the forward Euler discretization of the following ODE model, which depicts the evolution of $x_i$:
\begin{equation}
\label{eq:ode_resnet}
\frac{dx_i(t)}{dt}=f(x_i(t),\theta(t)), \quad x_i(0) = x_i.
\end{equation}
Time forms a continuous analogy to the layer index, where each layer corresponds to an iteration of the evolution. This ODE contains only the convection term, and each point moves independently without collision. To enhance the interactions among data points especially for unlabeled samples, we introduce an additional diffusion term in~\eqref{eq:ode_resnet}, which leads to the following convection-diffusion ODE system:
\begin{equation}
\label{eq:ode_resnet_diffusion}
\frac{d x_i(t)}{dt}=f(x_i(t),\theta(t))-\gamma\sum_{j=1}^N w_{ij}(x_i(t)-x_j(t)),
\end{equation}
for all $i=1,2,\ldots,N$, where $N$ is the number of points, $\gamma>0$ is a parameter that controls the strength of diffusion and $w_{ij}\geq 0$ is the weight between $x_i$ and $x_j$. By designing a weight matrix that can depict the similarity between points, we can expect similar points are brought together, while dissimilar points are driven apart. In this paper, the convection term $f$ is set to be a simple 2-layer network with width $w$, i.e.\
\begin{equation}
	f(x(t),\theta(t))=\sum_{i=1}^{w}a_t^{(i)}\sigma(\w_t^{(i)}\cdot x(t)+b_t^{(i)}).
\end{equation}
Here $x(t)\in \mathbb{R}^d$, $b_{t}^{(i)}\in \mathbb{R}$, $a_{t}^{(i)},~\w_{t}^{(i)}\in \mathbb{R}^d$ and $f: \mathbb{R}^d \rightarrow \mathbb{R}^d$.
The activation function $\sigma(\cdot)$ is chosen to be ReLU. $\theta(t) = [\w_{t}^{(i)}, b_t^{(i)}, a_t^{(i)}]_{i=1}^w$ is the collection of network weights at time $t$. In the next section, we will derive a practical algorithm based on the new ODE equation~\eqref{eq:ode_resnet_diffusion}.

%

\subsection{Algorithm}
\label{section:algorithm}

We discretize the convection-diffusion equation \eqref{eq:ode_resnet_diffusion} using the classic Lie-Trotter splitting scheme \cite{geiser2009decomposition}. After absorbing the time step $\Delta t$ into $f$ and $\gamma$, it leads to
\begin{align}
x_{i}^{k+1/2} & =x_{i}^k+f(x_i^k, \theta^k),\label{eq:convection_step}\\
x_i^{k+1} & = x_{i}^{k+1/2} - \gamma\sum_{j=1}^{N}w_{ij}(x_i^{k+1/2}-x_j^{k+1/2}). \label{eq:diffusion_step}
\end{align}
The convection step~\eqref{eq:convection_step} is nearly identical to the residual block~\eqref{eq:residual_block}, only differs in the time step, which is not essential as the implementation is the same. The added diffusion step~\eqref{eq:diffusion_step} can be seen as the stabilization of the convection step~\eqref{eq:convection_step}. If the weight matrix is pre-computed, the diffusion step is parameter free, thus the proposed diffusion term can be easily combined with any existing networks or algorithms in a plug-and-play manner. To construct the weight matrix, we use the Gaussian kernel $k(x,y)=\exp(-\|x-y\|_2^2/\sigma^2)$ to measure the similarity between data points. $\sigma$ is a parameter to adjust the distribution of weight. Next, we introduce two operators, $\mathrm{Sparse}$ and $\mathrm{Normalize}$, and one hyperparameter, $n_{\mathrm{top}}$, to obtain a sparse and balanced weight matrix. $\mathrm{Sparse}$ is a truncation operator to make the weight matrix sparse. In each row, it keeps the largest $n_{\mathrm{top}}$ entries and truncate other entries to 0. $\mathrm{Normalize}$ symmetrically normalize the weight matrix. Once constructed, the weight matrix remains unchanged during the training process. 

Using the Lie-Trotter scheme, we get one diffusion step \eqref{eq:diffusion_step} after convection step \eqref{eq:convection_step}. However, in our implementation, there are often several diffusion steps followed by each convection step. The reason is that the diffusion term has strong numerical stiffness as proved in Appendix A. The step size $\gamma$ should be small enough to keep numerical stability when the simple explicit Euler discretization method is used. Consequently, in order to maintain certain diffusion strength, we will use simple forward Euler scheme to discretize the diffusion term. Moreover, even if the total strength is small, multiple diffusion layers also give slightly better results in experiments. Thus, in the networks, we add $r$ diffusion layers after each residual block, each with a fixed step size $\gamma$. The illustration of our Diff-ResNet can be found in Figure~\ref{fig:network}. We summarize our method in Algorithm \ref{alg}.

\begin{algorithm}[hbtp!]
	\begin{algorithmic}[1]
		\State {\bfseries Input:} Labeled data points $\{(x_i,y_i)\}_{i=1}^{N_1}$. Unlabeled data points $\{x_j\}_{j=1}^{N_2}$. Number of blocks $s$. Number of diffusion steps $r$. Step size $\gamma$.
		\State {\bfseries Output:} Trained network parameters $\{\theta^k\}$
		\item[]
		\State Construct weight matrix $\tilde W$ by $\tilde w_{ij}= \exp(-\|x_i-x_j\|_2^2/\sigma^2)$ for all $i,j\in[N]$
		\State $W= \mathrm{Normalize}(\mathrm{Sparse} (\tilde W, n_{\mathrm{top}}))$
		\While{epoch $\leqslant$ MAX\_ITER}
		\State $x_i^0=x_i(0)=x_i$
		\For{$k=0, 1, \cdots, s-1$}
		\State $x_{i}^{k+1/2} =x_{i}^k+f(x_i^k, \theta^k)$\Comment{Convection Step}
		\For{$m=0, 1, \cdots, r-1$}
		\State $x_i^{k+1/2}=x_{i}^{k+1/2} - \gamma \sum_{j=1}^{N_1+N_2} w_{ij}(x_i^{k+1/2}-x_j^{k+1/2})$ \Comment{Diffusion Step}
		\EndFor
		\State $x_i^{k+1} = x_{i}^{k+1/2}$
		\EndFor
		\State $x_i(1)=x_i^s$
		\State Feed $x_i(1)$ into a classification layer, compute loss function using $\{y_i\}_{i=1}^I$, back propagate, and update $\{\theta_k\}$ using gradient descent.
		\State epoch = epoch$+1$
		\EndWhile
	\end{algorithmic}
	\caption{Training algorithm for Diff-ResNet}
	\label{alg}
\end{algorithm}



\begin{Remark}
In diffusion step \eqref{eq:diffusion_step}, the feature map of the $i$-th data point depends on the feature map of all data points at previous layer, which is not realistic in tasks when the total number of data points is too large. In our implementation, we adopt the mini-batch training strategy. That is, the weights in each batch are sparsified and normalized correspondingly. 
\end{Remark}

%% file: data/theory.tex
\begin{figure}[hbtp!]
	\centering
	\includegraphics[width=0.7\linewidth]{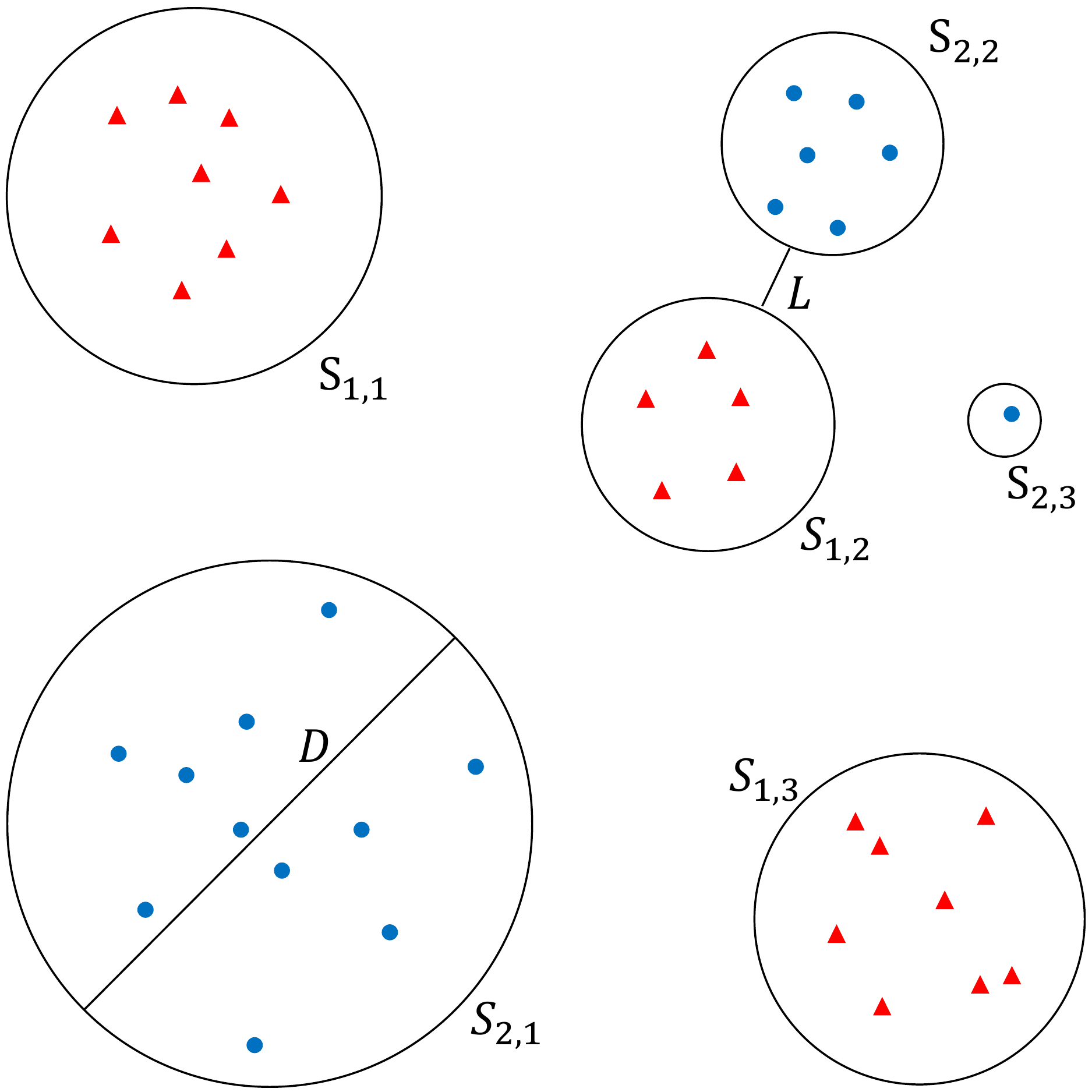}
	\caption{Illustration of Structured Data Assumption: $S_{i,j}$ stands for different subsets. $D$ is the upper bound of diameters of subclasses and $L$ is the lower bound of distances among subclasses.}
\end{figure}
In this section, the effectiveness of diffusion mechanism will be analyzed in theory. For the sake of simplicity, we only consider the binary classification problem and our analysis can be naturally extended to multi-class case.

\subsection{Structured Data Assumption}

Our dataset is generated as follows. Suppose all the data points come from $S=\coprod_{i=1}^k S_i$. Symbol $\coprod$ means that $S=\bigcup_{i=1}^k S_i$ and $S_{i_1}\bigcap S_{i_2}=\varnothing,~\forall i_1\neq i_2$. Each set $S_i$ contains the points from the $i$-th class. We further assume that each set $S_i$ can be divided into several non-overlap and bounded subsets $S_i=\coprod_{j=1}^l S_{i,j}$, each $S_{i,j}$ corresponding to a subclass. In the binary classification, $k=2$ and $S= S_1 \coprod S_2$. $l$ is the number of subclasses, which may vary with class. However, we can set $l$ to be the maximum across all classes, and let $S_{i,j} =\varnothing $ for those nonexistent subclasses. Denote $M=kl$ as the total number of subsets. The distance between two disjoint sets $A$,$B$ is defined as $$\mathrm{dist}(A,B)= \inf_{x\in A, y\in B} \|x-y\|^2.$$
The diameter of a set $A$ is defined as
\[\mathrm{diam}(A)= \sup_{x,y\in A}\|x-y\|^2.\]
\begin{Remark}
   The reason we introduce subclass instead of directly using class is that it can relieve our separability assumption. We do not need two classes to be well apart, which is not realistic in real-world scenario. Rather, we only need local subclasses to form clusters.
\end{Remark}
We are now ready to state the structured data assumption. 

\noindent(A) (Upper Bound of Diameters) There exists $D>0$ such that for each $(i,j)\in [k]\times[l]$, we have: $S_{i,j}\in B(x_{0},D/2)$\footnote{$B(x_{0},D/2)$ is defined as a ball centered at $x_0$ with radius $D/2$} for some $x_{0}$ , then:
\[\mathrm{diam}(S_{i,j}) \leqslant D.\]
\noindent(B) (Lower Bound of Distances) There exists $L>0$ such that for each $(i_1,j_1)\neq (i_2,j_2) \in [k]\times[l]$, we have:
\[\mathrm{dist}(S_{i_1,j_1}, S_{i_2,j_2}) \geqslant L.\]
Here, $L$ and $D$ are similar to inter-class distance and intra-class distance, which are terminologies widely used in the field of clustering. The difference lies in that the diameter used in our analysis is the upper bound for the points in the subclass, corresponding the local intra-class distance. In this sense, the Structured Data Assumption may be more practical for dealing with complex datasets. For example, in MNIST dataset, every digit number may have different handwriting styles, which corresponds to different subclasses and fits to our analysis framework. The intuition behind the Structured Data Assumption is simple: similar samples should be close while dissimilar samples should be far away.

\subsection{Theoretical Analysis}
\label{subsec:theoretical}
We present the theoretical results of this paper so as to explain the role of diffusion mechanism in binary classification. Due to the space limit, we defer all the proofs to the Appendix.

\begin{Definition}
A set $\{(x_{i},y_{i})\}_{i=1}^N,~x_{i}\in R^{d},~y_{i}\in [k]$ is called \emph{linear separable} if and only if there exists a hyperplane that cuts the full space $R^{d}$ into 2 half-spaces and data points in each half-space have a common corresponding label.
\end{Definition}

\begin{Theorem}
	\label{thm:1}
	(Approximation Property of ResNet Flow) If all the $S_{i,j},~(i,j)\in [k]\times[l]$ can be separated by a set of $M-1$ parallel hyperplanes, i.e., there exists a unique $S_{i,j}$ that lies in the region between each pair of adjacent parallel hyperplanes. Then we can construct the time-dependent parameters $\theta(t) = [\w_{t}^{(i)}, b_t^{(i)}, a_{t}^{(i)}]_{i=1}^w$ in the ResNet flow:
	\[f(x(t),\theta(t))=\sum_{i=1}^{w}a_{t}^{(i)}\sigma(\w_{t}^{(i)}\cdot x(t)+b_{t}^{(i)})\]
	such that all the final step\footnote{Time only forms a continuous analogy to the layer index, where each layer corresponds to forward propagation of the flow. Without loss of generality, we assume the final time step is $T=1$. } regions $F_{i}=\{x(1): x(0)\in S_{i,j},j\in[l]\}, i\in [k]$  are linear separable. We need $2M+O(d)$ different variables and $M/w$ layers.
\end{Theorem}

We give a sketch of proof. Consider the simplest case in which each $S_{i,j}$ only contains one point and the width $w$ is also 1. Our main idea is to construct a ResNet flow such that each subclass is moved to a proper position with better separability. We split the total time into $N$ intervals and deal with points one by one. After solving the simplest case, we extend to case $w>1$, i.e., the network width is larger. Lastly, we prove the case when there are multiple points in each subclass $S_{i,j}$.

In the classical XOR dataset, the original data points $x_{i}(0) = x_{i}$ are not linear separable. However, Theorem \ref{thm:1} tells us that: through the ODE flow, we can make the output features $x_{i}(1)$ become linear separable, so that a proper fully-connected layer can achieve accurate classification. 

Our next step is to show that the condition in Theorem~\ref{thm:1} can be satisfied by introducing diffusion mechanism. First, we give a sufficient condition that is related to the Distance-Diameter ratio.
\begin{Theorem}
	\label{thm:2}
	If the Distance-Diameter Ratio is large enough:
	\[\frac{L}{D}\geqslant \frac{M(M-1)\sqrt{\pi}}{4}d,\]
	then all the $S_{i,j},~(i,j)\in [k]\times[l]$ can be separated by a set of $M-1$ parallel hyperplanes.
\end{Theorem}

This proof relies on comparing the surface area of a specific set with the unit sphere. It is noted that $M$ is the number of subclasses that has $M \ll N$ in most cases. Thus the constant in the inequality is achievable. The next proposition shows that the diffusion step can increase this ratio with exponential rate.

The diffusion of each data point $x_i$ is modeled as
\[\frac{d x_i(t)}{dt}=-\gamma\sum_{j=1}^N w_{ij}(x_i(t)-x_j(t)), x_i(0)=x_i\]
for $i\in [N]$. Thus, all points change their positions subject to mutual interactions, and the distances between subsets and diameters of subsets are changed accordingly. Let $S_{i,j}(t)$ be the subset at time $t$, we define lower bound of distances $L(t)$ and upper bound of diameters $D(t)$ at time $t$ as
\[\mathrm{diam}(S_{i,j}(t)) \leqslant D(t),~\forall(i,j)\in [k]\times[l],\]
\[\mathrm{dist}(S_{i_1,j_1}(t), S_{i_2,j_2}(t)) \geqslant L(t),~\forall (i_1,j_1)\neq (i_2,j_2) \in [k]\times[l].\]
Let $G=(V,E)$ be a graph, where $V$ is the set of data points, and $E$ is the set of edges corresponding to non-zero weights $w_{ij}$. Then we have the following proposition describing the diffusion effects.
\begin{Proposition}
	\label{prop:1}
	 Suppose the data points in each subset $S_{i,j},~(i,j)\in [k]\times[l]$ form a connected component in the graph $G$, and each $S_{i,j}$ is convex. Then, the Distance-Diameter Ratio grows to infinity, i.e.\
	 \[\lim_{t\rightarrow \infty}\frac{L(t)}{D(t)}=\infty.\]
	 Moreover, the growth rate is exponential.
\end{Proposition}

The basic idea for proving Proposition \ref{prop:1} is to show that $L(t)$ is nonincreasing while $D(t)$ converges to zero at exponential rate. Using the spectral clustering theory, it is proved that each subclass converges to its center along with the diffusion process.

To meet the assumption that points in each subset form a connected component in the graph, we should ensure (1) there is no edge that connect points among different subsets (2) any two vertices in the same subset $S_{i,j}$ are connected to each other. By the construction of weight matrix, each vertex in graph $G$ is only connected to its $n_{\mathrm{top}}$ nearest neighbors. Thus the first argument is satisfied when the nearest neighbors only contain points from the same subclass, which requires that $n_{\mathrm{top}}$ should not be too large. On the other hand, the threshold for the connectivity of a k-nearest neighbor graph is $O(\mathrm{log} n)$~\cite{balister2005connectivity}, where $n$ in our setting should be the number of points in each subset.

The above analysis reveals that the diffusion mechanism is helpful for organizing data points by making data points from the same subclass region closer to each other while others relatively further away. As the Distance-Diameter ratio increases, it is easy for distinguishing data points using ResNet flow. This property is important for SSL/FSL problems as it deeply explores the relationship among points.

%% file: data/synthetic.tex
In this section, we show the efficacy of diffusion mechanism on synthetic data, and report the performance of the Diff-ResNet on semi-supervised graph learning and few-shot learning tasks. \footnote{Code at \url{https://github.com/shwangtangjun/Diff-ResNet}.}

\subsection{Synthetic Data}
We conduct experiments on four classical synthetic datasets: XOR, moon, circle and spiral. In XOR dataset, we directly apply diffusion without any convection. Then we can clearly see the evolution process of points that verifies the Proposition \ref{prop:1}. The other three datasets are used to show the effectiveness of diffusion in classification tasks. In this section, we only show results of XOR and circle datasets. Due to the space limitation, please refer to Appendix E.1.4 for more results.

We randomly collect 100 points each in four circles centered at (0,0), (0,2), (2,0), (2,2), respectively, with radius 0.75. These four circles are treated as the subsets corresponding to four subclasses. The circles centered at (0,0) and (2,2) belong to the same class, and points from them are colored red. The blue ones are generated similarly. Here, we show the evolution of points as diffusion strength goes to infinity. As stated in the Section~\ref{section:algorithm}, we stack diffusion steps with small step size $\gamma$ to ensure stability. In Figure~\ref{fig:XOR} , points distribution after 1, 10, 20 and 200 diffusion steps are given. 
\begin{figure}[hbtp]
	\subfloat[raw]{\raisebox{1.35ex}{\includegraphics[width=0.455\linewidth]{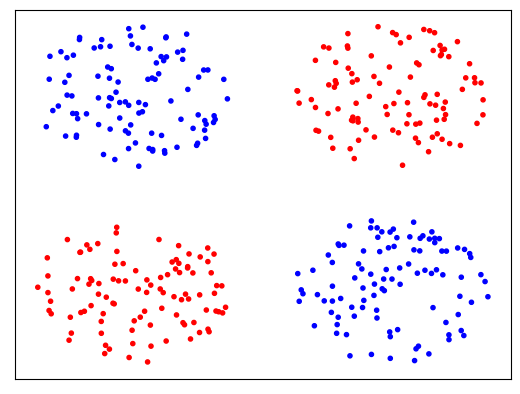}}}
	\hfill
	\subfloat[D and L/D]{\includegraphics[width=0.52\linewidth]{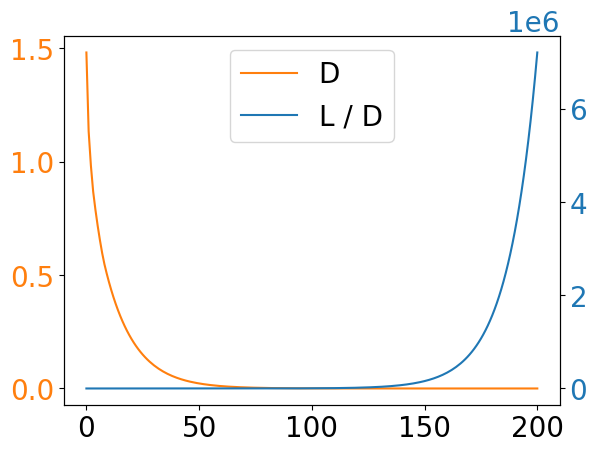}}
		
	\subfloat[step=1]{\includegraphics[width=0.25\linewidth]{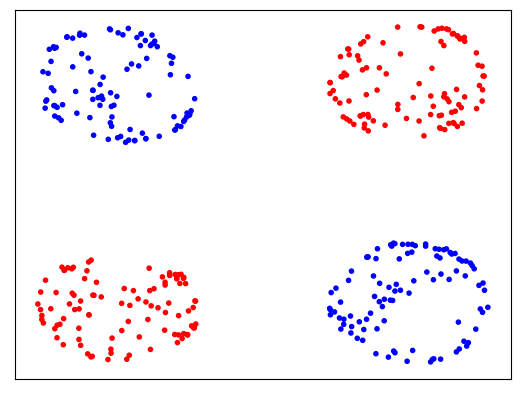}}
	\subfloat[step=10]{\includegraphics[width=0.25\linewidth]{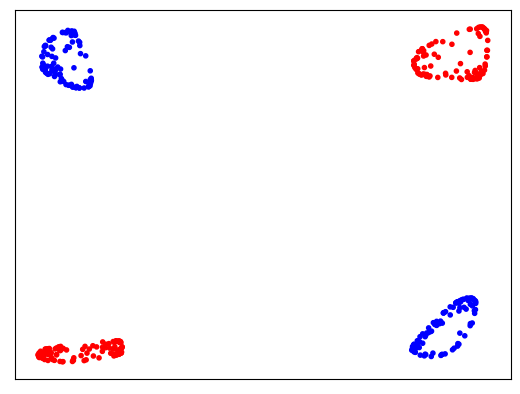}}
	\subfloat[step=20]{\includegraphics[width=0.25\linewidth]{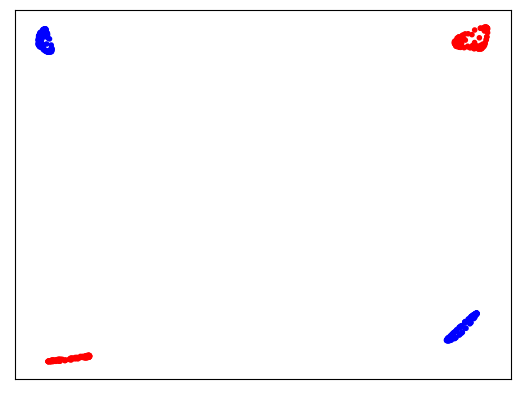}}
	\subfloat[step=200]{\includegraphics[width=0.25\linewidth]{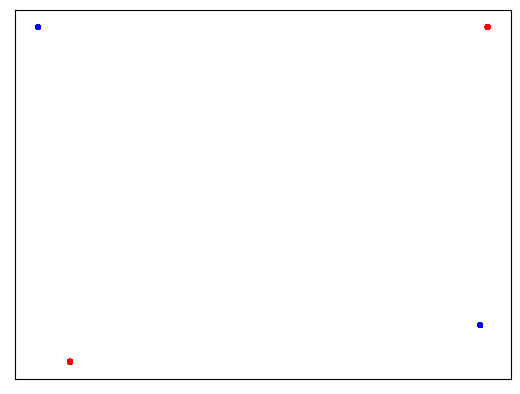}}
	\caption{Visualization of diffusion Mechanism on XOR dataset. (a) is raw data. (b) shows the evolution of D and L/D with diffusion steps. 
	(c)(d)(e)(f) depicts the evolution of points.\label{fig:XOR}}
\end{figure}
In the above example, the initial diameter is $D=1.5$ while the distance is $L=0.5$, which does not meet the sufficient condition $L>D$ in Proposition 1. However, as shown in Figure~\ref{fig:XOR}, this diffusion still works well. Data points in same subclass converge to a single point. We also observe from Figure~\ref{fig:XOR} (b) that the Distance-Diameter Ratio indeed grows exponentially to infinity. 

\begin{Remark}
    Some may doubt the use of terminology, diffusion, as it actually draws similar points together and create high density regions visually. However, the phenomenon shown in ~\ref{fig:XOR} is not contradictory to the definition of diffusion. The energy of a point is represented by its coordinate. We expect that neighboring elements in the graph will exchange energy until that energy is spread out evenly throughout all of the elements that are connected to each other. As a result, the diffusion mechanism acts as gathering points together.
    
\end{Remark}

Next, we show the effectiveness of diffusion in residual networks on binary classification tasks containing 1000 planar data points forming two circles. Two classes are marked with different colors. We use residual networks with hidden dimensions 2 (so that it will be convenient for us to visualize the features). The details of experiment settings can be found in Appendix E.1. During training residual networks with or without diffusion mechanism, we plot the features before the final classification layer in Figure~\ref{fig:two_circle}. Note that what we plot are not the input data points. Thus, even without diffusion, the points have to pass through a randomly initialized residual block. So in subfigure (c) of Figure~\ref{fig:two_circle}, features are different from raw input points in (a). The results of circle dataset is shown in Figure~\ref{fig:two_circle}.

\begin{figure}[hbtp]
	\centering
	\subfloat[raw]{\raisebox{1.6ex}{\includegraphics[width=0.45\linewidth]{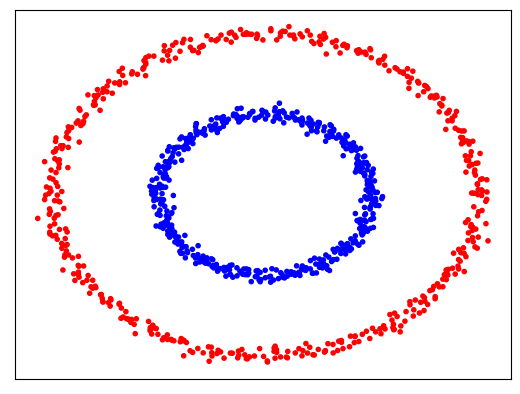}}}
	\hfill
	\subfloat[accuracy]{\includegraphics[width=0.5\linewidth]{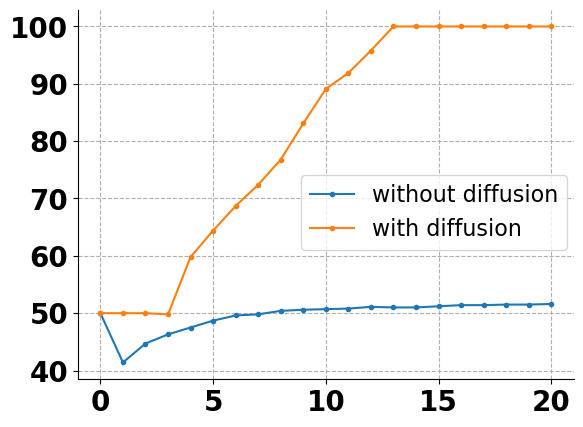}}
	
	\subfloat[w/o, epoch=0]{\includegraphics[width=0.33\linewidth]{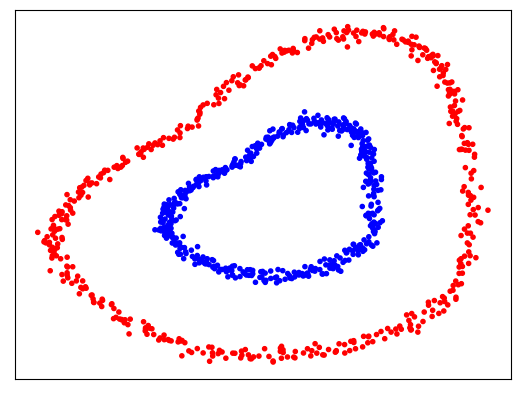}}
	\hfill
	\subfloat[w/o, epoch=10]{\includegraphics[width=0.33\linewidth]{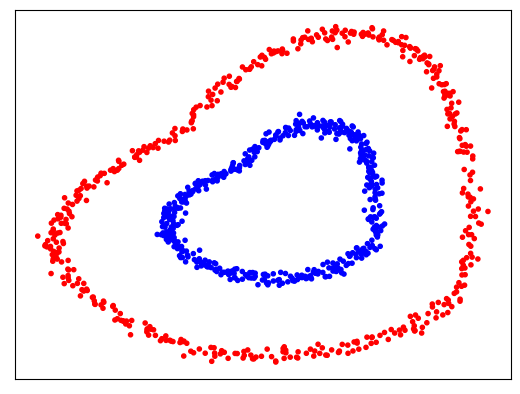}}
	\hfill
	\subfloat[w/o, epoch=20]{\includegraphics[width=0.33\linewidth]{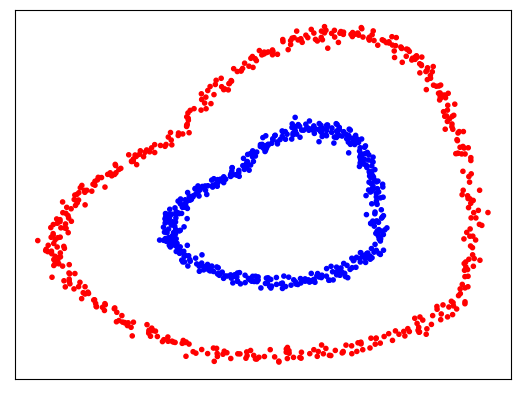}}
	
	\subfloat[w, epoch=0]{\includegraphics[width=0.33\linewidth]{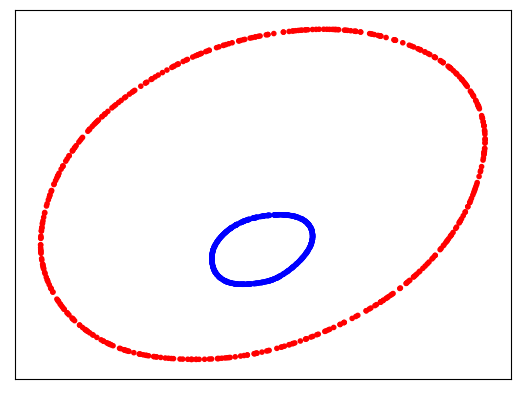}}
	\hfill
	\subfloat[w, epoch=10]{\includegraphics[width=0.33\linewidth]{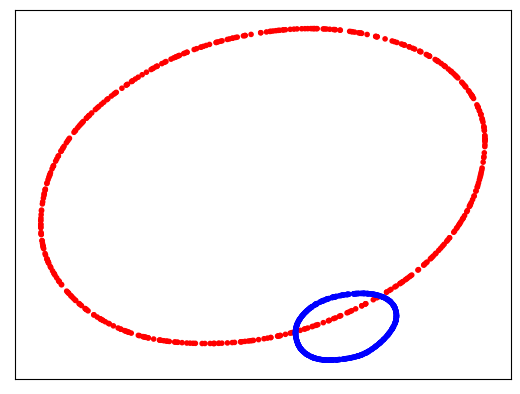}}
	\hfill
	\subfloat[w, epoch=20]{\includegraphics[width=0.33\linewidth]{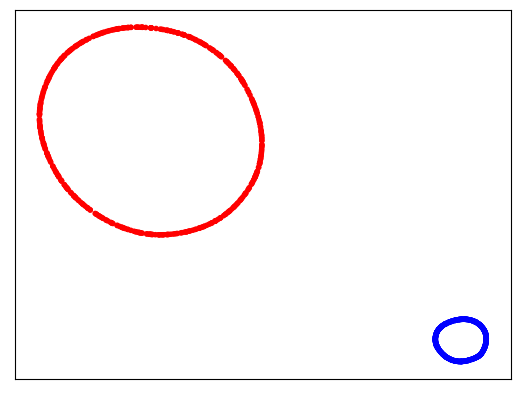}}
	\caption{ResNet and DiffResNet on circle dataset, (a) shows the position of raw data points. (b) figures the accuracy of classification tasks with the training epoch. (c)(d)(e) are features before the final classification layer without diffusion from different epochs. (f)(g)(h) are features with diffusion.}
	\label{fig:two_circle}
\end{figure}

As shown in Figure~\ref{fig:two_circle}, diffusion can reduce the noise. In Figure~\ref{fig:two_circle} (f),(g),(h), the features are very clean while in Figure~\ref{fig:two_circle} (f),(g),(h), features are still noisy. Moreover, diffusion step makes the final feature much easier to separate. In Figure~\ref{fig:two_circle} (h), features can be easily separated by a straight line while the features are not linear separable without diffusion as shown in Figure~\ref{fig:two_circle} (e). It is not surprise that in this example ResNet fails to give correct classification considering it only has 18 parameters in total. With the help of diffusion step, even this small network with only 18 parameters can give correct classification which demonstrates that diffusion is very useful in classification problem.

%% file: data/graph.tex
\subsection{Graph Learning}
We investigate the effect of diffusion on semi-supervised learning problems in graph. In diffusion step, a key point is how to determine the weights that can properly depict the relationship between data points. Nonetheless, in graph this is not a problem since the weights have already been given in the form of adjacent matrix. We report results for the most widely used citation network benchmarks including Cora, Citeseer and Pubmed. These datasets are citation networks in which nodes are documents, edges are citation links and features are sparse bag-of-words vectors. The concrete dataset statistics is given in Appendix E.2.1. Moreover, rather than using the fixed Planetoid~\cite{yang2016revisiting} split, we follow~\cite{shchur2018pitfalls} and report results for all datasets using 100 random splits with 20 random initializations each.

The mainstream approach in graph learning, such as GCN~\cite{kipf2017semi}, GraphSAGE~\cite{hamilton2017inductive} and GAT~\cite{velickovic2018graph}, contains aggregation steps, which aggregate feature information from neighbors using the adjacent matrix, and then predict labels with aggregated information. Different from these conventional paradigm, our method is composed of convection-and-then-diffusion step. The convection step make full use of the label information, while the diffusion step exchange the feature information among data samples. The adjacent matrix is only used in the diffusion step.

We compare our method with several graph learning methods: three of the most popular architectures, GCN, GraphSAGE (its two variants), GAT, and recent ODE-based GNN architectures, Continuous Graph Neural Networks
(CGNN)~\cite{xhonneux2020continuous}, Graph Neural Ordinary
Differential Equations(GDE)~\cite{poli2019graph}, and Graph Neural Diffusion(GRAND)~\cite{chamberlain2021grand}. The detailed network structure and parameter settings can be found in Appendix E.2. The classification results are reported in Table~\ref{table:graph}. Diff-ResNet is significantly better than ResNet without diffusion(No-Diff-ResNet). It achieves more than 15\% accuracy boost on average, which is a strong evidence for the benefit of diffusion. Moreover, despite the large discrepancy between our diffusion network and mainstream networks, our method still achieves competitive results with respect to classical and recent methods in graph learning. Thus, we propose an alternative path for semi-supervised graph learning problems. 


\begin{table}[hbtp]
\caption{The mean accuracy and std (\%) of node classification over 100 random dataset splits and 20 random initializations each.}
	\centering
	\begin{tabular}{lccc}
		\toprule
		& Cora          & Citeseer      & Pubmed        \\ \midrule
        MLP & 58.2 $\pm$ 2.1 & 59.1 $\pm$ 2.3 & 70.0 $\pm$ 2.1 \\
		GCN\cite{kipf2017semi} & 81.5 $\pm$ 1.3         & 71.9 $\pm$ 1.9          & 77.8 $\pm$ 2.9        \\
		GraphSAGE-mean\cite{hamilton2017inductive} & 79.2 $\pm$ 7.7          & 71.6 $\pm$ 1.9          & 77.4 $\pm$ 2.2          \\
		GraphSAGE-maxpool\cite{hamilton2017inductive} & 76.6 $\pm$ 1.9         & 67.5 $\pm$ 2.3          & 76.1 $\pm$ 2.3          \\
		GAT\cite{velickovic2018graph}         & 81.8 $\pm$ 1.3        & 71.4 $\pm$ 1.9         & 78.7 $\pm$ 2.3          \\
        CGNN\cite{xhonneux2020continuous} & 81.4 $\pm$ 1.6 & 66.9 $\pm$ 1.8 & 66.6 $\pm$ 4.4 \\
        GDE\cite{poli2019graph} & 78.7 $\pm$ 2.2 & 71.8 $\pm$ 1.1 & 73.9 $\pm$ 3.7 \\
        GRAND\cite{chamberlain2021grand} & \textbf{83.6} $\pm$ 1.0 & 73.4 $\pm$ 0.5 & 78.8 $\pm$ 1.7\\
		No-Diff-ResNet & 58.9 $\pm$ 1.9 & 61.9 $\pm$ 2.1 & 70.1 $\pm$ 2.1 \\
		Diff-ResNet(ours) & 82.1 $\pm$ 1.1 & \textbf{74.6} $\pm$ 1.8 & \textbf{80.1} $\pm$ 2.0 \\ \bottomrule
	\end{tabular}
	\label{table:graph}
\end{table}

Additionally, it is reported that methods based on aggregation of neighboring information suffers over-smoothing problems with increasing depth~\cite{li2018deeper, chen2020simple}. As is observed in Figure~\ref{fig:graph}, the performance of GCN drops more than 50\% on average when the network depth increases to 32. Different from GCN, our Diff-ResNet does not use aggregation, thus the representations of the
nodes will not converge to a certain value and become indistinguishable. When the number of layers increases to 32, performance of Diff-ResNet only drops less than 10\%, which is partly due to the deep network training burden. This serves as an evidence that our network structure is far different from mainstream architectures.

\begin{figure}[hbtp!]
	\centering
	\includegraphics[width=0.9\linewidth]{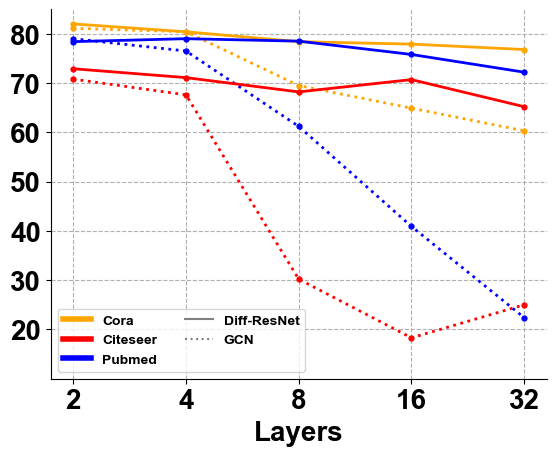}
	\caption{Performance of architectures of different depth. The x-axis represents number of layers, y-axis is the accuracy.}
	\label{fig:graph}
\end{figure}


%% file: data/fewshot.tex
\subsection{Few-shot Learning}
\label{section:fewshot}

\begin{table*}[hbtp]
	\begin{center}
	\caption{Ablation Study of Diffusion Mechanism. Table shows average classification accuracy (\%). \label{table:ablation}}
	\begin{normalsize}
    \begin{tabular}{cccccccc}
    \toprule
    & &\multicolumn{2}{c}{\textbf{\textit{mini}ImageNet}}&\multicolumn{2}{c}{\textbf{\textit{tiered}ImageNet}}&\multicolumn{2}{c}{\textbf{CUB}} \\
    \cmidrule{3-8}
    \textbf{Backbone}&\textbf{Method} & \textbf{1-shot}& \textbf{5-shot}& \textbf{1-shot}& \textbf{5-shot}& \textbf{1-shot}& \textbf{5-shot}\\
	\midrule
	\multirow{5}{*}{ResNet-18}& Nearest Prototype & 57.09 & 79.30 & 63.46 & 83.53 & 66.53 & 86.03\\
	& Diffusion & 57.81 & 79.54 & 64.38 & 83.83 & 67.96 & 86.49\\
	& Convection & 53.04 & 79.58 & 56.17 & 82.79 & 58.68 & 86.10\\
	& External Convection-Diffusion & 55.30 & 79.55 & 65.15 & 83.74 & 70.00 & 86.92\\
	& Internal Convection-Diffusion & \textbf{68.47} & \textbf{80.02} & \textbf{75.31} & \textbf{84.19} & \textbf{79.12} & \textbf{87.18}\\
	\midrule
	\multirow{5}{*}{WRN}& Nearest Prototype & 59.54 & 79.81 & 65.60 & 84.76 & 68.56 & 86.13 \\
	& Diffusion & 60.30 & 80.31 & 66.59 & 85.17 & 69.70 & 86.58\\
	& Convection & 57.07 & 80.92 & 58.79 & 84.69 & 62.84 & 87.51\\
	& External Convection-Diffusion & 59.46 & 80.70 & 68.10 & 85.31 & 73.08 & 87.08 \\
	& Internal Convection-Diffusion & \textbf{69.77} & \textbf{81.17} & \textbf{77.44} & \textbf{85.50} & \textbf{80.31} & \textbf{87.76}\\
	\bottomrule
	\end{tabular}
	\end{normalsize}
	\end{center}
\end{table*}

Given a dataset $\mathbb{X}= \mathbb{X}_{s}\cup\mathbb{X}_{q}$, where $\mathbb{X}_{s} = \{(x_i,y_i)\}_{i=1}^{N_1}$ is the support set with label information and $\mathbb{X}_{q} = \{x_j\}_{j=1}^{N_2}$ is the query set without labels, the goal of few-shot learning is to find the labels of points in the query set when the size of support set $|N_1|$ is very small. Among existing few-shot learning methods, embedding learning is a typical approach, which maps each sample to a low dimensional spaces such that similar samples are close while dissimilar samples are far away. The embedding function can be learned by a deep neural network (a.k.a. backbone), which is pretrained using a large number of labeled examples over base classes. In the few-shot learning problems, the pretrained embedding function is fixed and maps all data samples into the embedded space.

We conduct experiments on three benchmarks for few-shot image classification: \textit{mini}ImageNet, \textit{tiered}ImageNet and CUB. The \textit{mini}ImageNet and \textit{tiered}ImageNet are both subsets of the larger ILSVRC-12 dataset~\cite{russakovsky2015imagenet}, with 100 classes and 608 classes respectively. CUB-200-2011~\cite{wah2011caltech} is another fine-grained image classification dataset with 200 classes. We follow the standard dataset split as in previous papers~\cite{Sachin2017,wang2019simpleshot, chen2019closerfewshot}. All images are resized to $84 \times 84$, following~\cite{vinyals2016matching}.

We choose two widely used networks, ResNet-18~\cite{he2016deep} and WRN-28-10~\cite{Zagoruyko2016WRN} as our backbone: the latter widens the residual blocks by adding more convolutional layers (28 layers) and feature planes (10 times). First, we train the backbone on the base classes using cross-entropy loss with label smoothing factor of 0.1, SGD optimizer, standard data augmentation and a mini-batch size of 256 to train all models. Note that our training procedure does not involve any meta-learning or episodic-training strategy. The model is trained for $T=100$ epochs for \textit{mini}ImageNet and \textit{tiered}ImageNet, and $T=400$ epochs for CUB due to its small size. We use a multi-step scheduler, which decays the learning rate by 0.1 at $0.5T$ and $0.75T$. We evaluate the nearest-prototype classification accuracy on the validation set and obtain the best model. The embedding training process is in general similar to that in SimpleShot~\cite{wang2019simpleshot} and LaplacianShot~\cite{ziko2020laplacian}, but details are slightly different. Eventually we get an embedding function which maps the original data point to $\mathbb{R}^M$ where $M=512$ for ResNet-18 and $M=640$ for WRN-28-10.


After we obtain a feature vector for every data point, we compare the performance of 5 typical classification methods to emphasize the effectiveness of diffusion mechanism.

\textbf{(1) Nearest Prototype}. The prototype $m_c$ of each class $c$ is the average of support set $\mathbb{X}_s^c$
$$m_c=\frac{1}{|\mathbb{X}_s^c|}\sum_{x\in \mathbb{X}_s^c} x$$
Then the query sample is classified as class $c$ if it is closest to prototype $m_c$ in Euclidean distance. It is the most natural classification method, and serves as a baseline.

\textbf{(2) Diffusion}. We try to minimize an objective function with Laplacian regularizer in an iterative way. In LaplacianShot~\cite{ziko2020laplacian}, the author optimizes the loss function below
$$\mathcal{L}= \sum_{i=1}^{N_2} \sum_{c=1}^C y_{i,c} d(x_i-m_c)+ \frac{\lambda}{2} \sum_{i,j=1}^{N_2} w\left(\mathbf{x}_{i}, \mathbf{x}_{j}\right)\left\|\mathbf{y}_{i}-\mathbf{y}_{j}\right\|^{2}$$
$N_2=|\mathbb{X}_q|$ is the number of query samples. $\mathbf{y}_{i}=[y_{i,1},\cdots,y_{i,C}]\in \{0,1\}^C$ is in the $C$-dimensional simplex, which assigns label to each query point. $d$ is Euclidean distance. $w\left(\mathbf{x}_{i}, \mathbf{x}_{j}\right)$ is the weight between $\mathbf{x}_{i}$ and $\mathbf{x}_{j}$.

The first loss term functions similar to nearest prototype classification. The second loss term is the well-known Laplacian regularizer. We use the iterative algorithm provided by LaplacianShot~\cite{ziko2020laplacian} to minimize the objective function $\min_{\mathbf{y}_{i}} \mathcal{L}$. Since there is no neural network, or manipulations on features, rather just label propagation, we name this method as \textbf{Diffusion}. 

\textbf{(3) Convection.} We minimize cross entropy loss
$$\mathcal{L}= -\sum_{i=1}^{N_1} \sum_{c=1}^C y_{i,c} \mathrm{log}(f(x_i)_c)$$
on the support set using gradient descent and a simple 2-layer residual network $f$. $N_1=| \mathbb{X}_{s}|$ is the number of support samples. The detailed network structure can be found in Appendix E.3.2. There is no relationship between data points during training, and the residual network is the counterpart of a convection ODE, so we refer to this method as \textbf{Convection}. 

\textbf{(4) External Convection-Diffusion.} We minimize cross entropy loss on the support set plus Laplacian regularizer using gradient descent and simple 2-layer residual network. The difference from \textbf{(3)} is that we add a Laplacian regularizer
\begin{align*}
    \mathcal{L}= -&\sum_{i=1}^{N_1} \sum_{c=1}^C y_{i,c} \mathrm{log}(f(x_i)_c)\\ +&\frac{\mu}{2}\sum_{i,j=1}^{N} w\left(\mathbf{x}_{i}, \mathbf{x}_{j}\right)\left\|f(x_i)-f(x_j)\right\|^{2}
\end{align*}
to the loss function. $N=N_1+N_2=| \mathbb{X}_{s}\cup\mathbb{X}_{q}|$ is the total number of support samples and query samples. The first variation of the Laplacian term coincides with the diffusion term of our convection-diffusion ODE. The residual network structure corresponds to convection, while the Laplacian regularizer corresponds to diffusion. As diffusion appears in the loss function externally, we refer to this method as \textbf{External Convection-Diffusion}. 

\textbf{(5) Internal Convection-Diffusion.} We minimize cross entropy loss
$$\mathcal{L}= -\sum_{i=1}^{N_1} \sum_{c=1}^C y_{i,c} \mathrm{log}(f(x_i)_c)$$
on the support set using gradient descent and simple 2-layer \textbf{diffusion} residual network (Diff-ResNet). The loss term is the same as \textbf{(3)}, while the difference is that we add diffusion layers internally in the network structure. By comparing \textbf{(4)} and \textbf{(5)}, we want to verify the necessity of incorporating diffusion as part of network structure, rather than as part of loss function.

Following the standard evaluation protocol~\cite{wang2019simpleshot}, we randomly sample 1000 5-way-1-shot and 5-way-5-shot classification tasks from the test classes, with 15 query samples in each class, and report the average accuracy of 5 methods above in Table \ref{table:ablation}. \textbf{Internal Convection-Diffusion}, which uses our proposed Diff-ResNet, achieves best results in all tasks. In 1-shot tasks, it has a performance boost of nearly 20\% compared to \textbf{Convection}, which clearly states the effectiveness of diffusion. Additionally, compared to \textbf{External Convection-Diffusion}, it also has approximately 10\% increase, indicating that embedding diffusion in the network structure is far more efficient than adding diffusion in the loss term. In 5-shot tasks, as the nearest prototype method has already provided competitive baseline, the increase is not remarkable, but still has about 1\% improvement against \textbf{Convection}.

We use T-SNE~\cite{van2008visualizing} to visualize the features before and after diffusion in 1-shot and 5-shot task in Figure~\ref{fig:tsne}. Labeled data are marked as stars, and unlabeled data are marked as circles. In 1-shot scenario, we can observe that points are hard to separate at first. However, with the help of diffusion mechanism, points in the same subclass are driven closer, making it easier to classify. In 5-shot tasks, since the sample points already have nice separability, the improvement with diffusion mechanism is not so remarked as that in 1-shot tasks. Nonetheless, we could verify the necessity of introducing subclass in Structured Data Assumption, as the blue points are indeed divided into two subsets.

\begin{figure}[hbtp]
	\centering
	\subfloat[1-shot, before]{\includegraphics[width=0.5\linewidth]{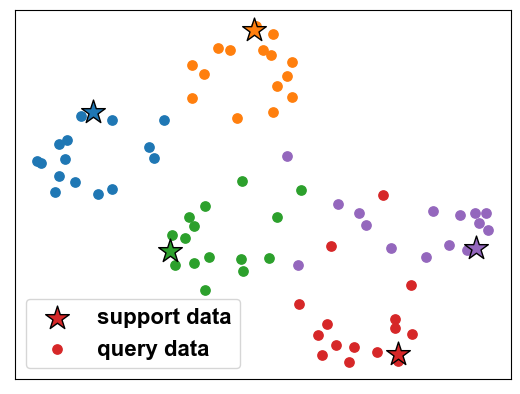}}
	\hfill
	\subfloat[1-shot, after]{\includegraphics[width=0.5\linewidth]{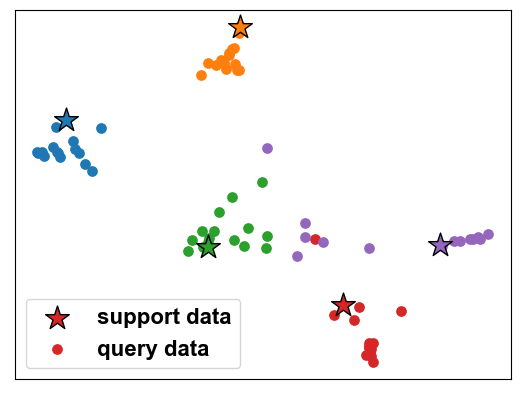}}
	
	\subfloat[5-shot, before]{\includegraphics[width=0.5\linewidth]{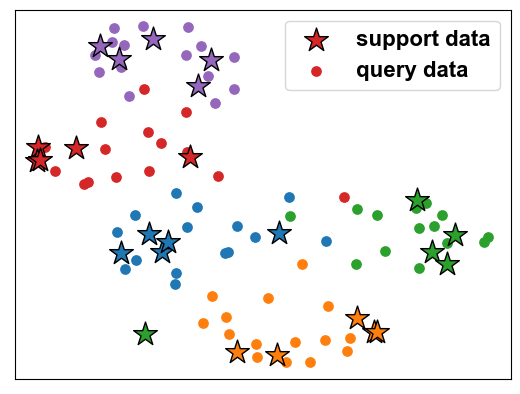}}
	\hfill
	\subfloat[5-shot, after]{\includegraphics[width=0.5\linewidth]{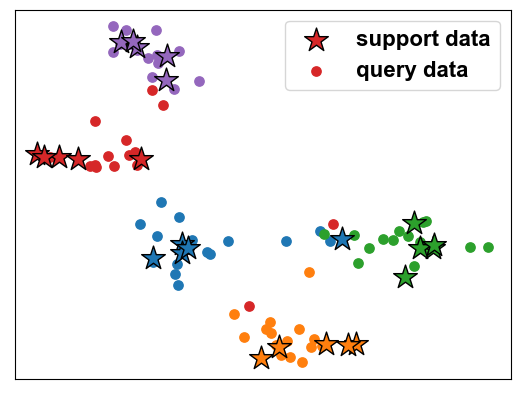}}
	\caption{T-SNE visualization of features before and after diffusion steps. Stars represent labeled points(support data). Circles represent unlabeled points(query data).}
	\label{fig:tsne}
\end{figure}

\begin{figure}[hbtp]
	\centering
	\subfloat[ResNet-18]{\includegraphics[width=0.5\linewidth]{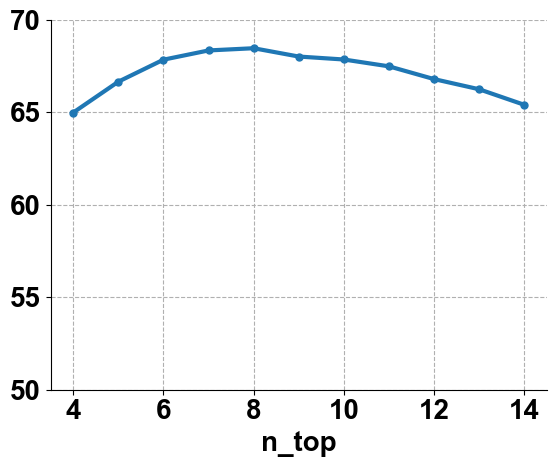}}
	\hfill
	\subfloat[WRN]{\includegraphics[width=0.5\linewidth]{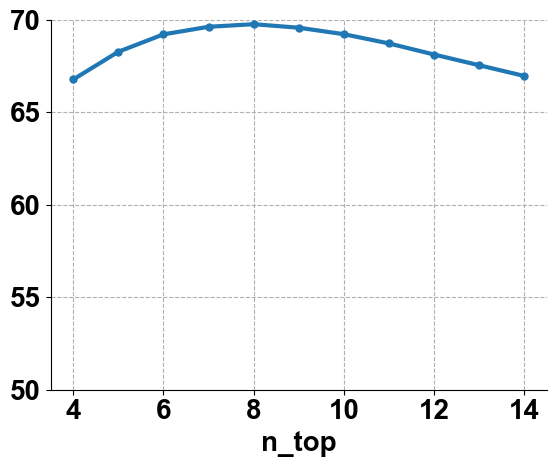}}
	
	\caption{The effect of $n_{\mathrm{top}}$ on \textit{mini}ImageNet with ResNet-18 and WRN as backbone. The x-axis represents $n_{\mathrm{top}}$, y-axis is the accuracy.}
	\label{fig:acc_ntop}
\end{figure}

Furthermore, we study the effect of several important parameters in diffusion mechanism: weight truncation parameter $n_{\mathrm{top}}$, diffusion step number $r$ and step size $\gamma$. We conduct experiments on 1000 5-way-1-shot tasks on \textit{mini}ImageNet with ResNet-18 and WRN as backbone, and report the average accuracy with different parameters.

First, we tune $n_{\mathrm{top}}$ in the $\mathrm{Sparse}$ operator. We choose $\sigma=[n_{\mathrm{top}}/2]$\footnote{$\sigma(x_i) = k$ means $\sigma$ is chosen to be the $k$-th closest distance from a specific point $x_i$ , so it varies with points.}. The results are depicted in Figure~\ref{fig:acc_ntop}. From the figure, we notice that $n_{\mathrm{top}}$ should be neither too small nor too large. Small $n_{\mathrm{top}}$ may 
break up large local clusters, while large $n_{\mathrm{top}}$ will include points from different classes into neighborhood. However, the classification accuracy is not very sensitive to $n_{\mathrm{top}}$, compared to diffusion strength.

\begin{figure}[hbtp]
	\centering
	\subfloat[ResNet-18]{\includegraphics[width=0.5\linewidth]{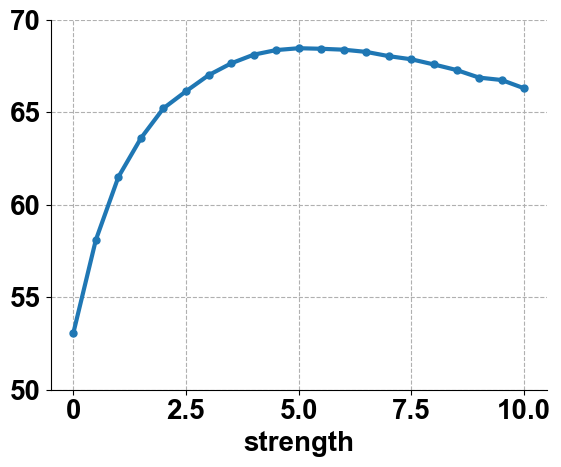}}
	\hfill
	\subfloat[WRN]{\includegraphics[width=0.5\linewidth]{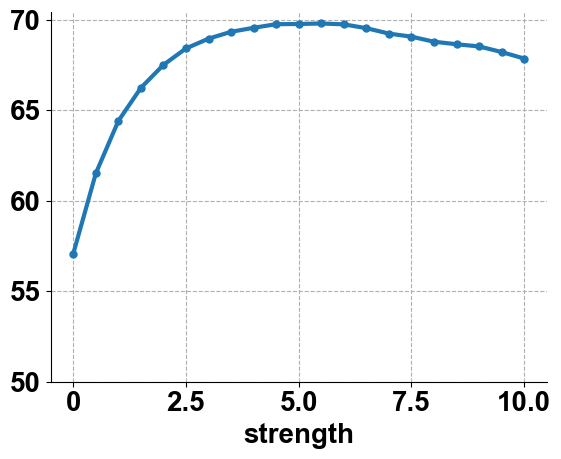}}
	
	\caption{The effect of total diffusion strength on \textit{mini}ImageNet with ResNet-18 and WRN as backbone. The x-axis is total strength $r\gamma$, y-axis is the accuracy.}
	\label{fig:acc_strength}
\end{figure}

Next, we study the effect of total diffusion strength, which is step size times step numbers $r\gamma$. We fix $\gamma=0.5$ and adjust $r$, ranging from 0 to 20. The results are shown in Figure~\ref{fig:acc_strength}. Based on our experiments, we should not push the strength to infinity as in synthetic data, because real data has much more complicated geometric structure such that we can not expect each class converge to a single point.

\begin{figure}[hbtp]
	\centering
	\subfloat[ResNet-18]{\includegraphics[width=0.5\linewidth]{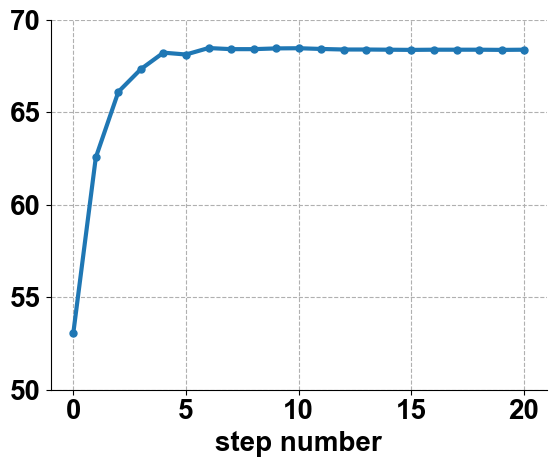}}
	\hfill
	\subfloat[WRN]{\includegraphics[width=0.5\linewidth]{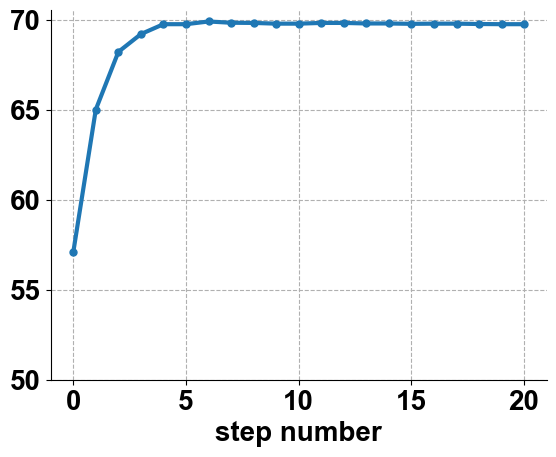}}
	
	\caption{The effect of number of steps $r$ on \textit{mini}ImageNet with ResNet-18 and WRN as backbone. The x-axis represents step number $r$, y-axis is the accuracy.}
	\label{fig:acc_step}
\end{figure}

Lastly, we fix total strength $r\gamma=5.0$, and change $r$ from 0 to 20, to study the effect of the step size. We require $\gamma \leqslant 1$ for stability. When the total diffusion strength is too large when $r < 5$, we set $\gamma=1.0$. As shown in Figure~\ref{fig:acc_step}, when the total diffusion strength is fixed, the accuracy almost keeps the same. Thus, stacking too many layers will not benefit.

To conclude, the performance of our Diff-ResNet mostly depend on the total diffusion strength $r\gamma$. With fixed strength, the number of diffusion layers $r$ and the truncation parameter $n_{\text{top}}$ has little effect on performance. We have briefly discussed the choice of $n_{\text{top}}$ in Subsection~\ref{subsec:theoretical} after Proposition~\ref{prop:1}, and state that a neither too large nor too small $n_{\text{top}}$ is better. Rather than cherry-picking $n_{\text{top}}$ according to different datasets or backbones, we fix $n_{\text{top}}=8$ and $\sigma=4$ in all few-shot learning experiments, which shows that a moderate $n_{\text{top}}$ is sufficient for good results. As for the number of diffusion layers $r$, the choice is made by fixing the diffusion step size $\gamma=0.5$, and choose the best diffusion strength on the validation set. From our experiment results, we observe that with fixed step size $\gamma$, the number of layers $r$ should be larger for 1-shot tasks, and smaller for 5-shot tasks. Also, the optimal diffusion strength varies with dataset, since the geometric property of each dataset is different. The detailed parameter choice is provided in Appendix E.3.

\begin{table*}[hbtp]
\caption{Average accuracy (in \%) and 95\% confidence interval in \textit{mini}ImageNet, \textit{tiered}ImageNet and CUB. We mark transductive learning method with \dag. *~denotes that the result is reimplemented using public official code with our pretrained backbone.}
\label{table:fewshot}
\begin{center}
\begin{small}
\begin{tabular}{lc|cc|cc|cc}
\toprule
& &\multicolumn{2}{c|}{\textbf{\textit{mini}ImageNet}}&\multicolumn{2}{c|}{\textbf{\textit{tiered}ImageNet}}&\multicolumn{2}{c}{\textbf{CUB}} \\
\cline{3-8}
\textbf{Methods} & \textbf{Backbone}&\textbf{1-shot}& \textbf{5-shot}& \textbf{1-shot}& \textbf{5-shot}& \textbf{1-shot}& \textbf{5-shot}\\
\hline
MAML \cite{finn2017model} & ResNet-18 & 49.61 $\pm$ 0.92 & 65.72 $\pm$ 0.77&-&-& 69.96 $\pm$ 1.01&82.70 $\pm$ 0.65\\
Baseline \cite{chen2019closerfewshot} & ResNet-18 & 51.87 $\pm$ 0.77 & 75.68 $\pm$ 0.63&-&-&67.02 $\pm$ 0.90 & 83.58 $\pm$ 0.54\\
RelationNet \cite{sung2018learning} & ResNet-18 & 52.48 $\pm$ 0.86 & 69.83 $\pm$ 0.68& 54.48 $\pm$ 0.93&71.32 $\pm$ 0.78& 67.59 $\pm$ 1.02 & 82.75 $\pm$ 0.58\\
MatchingNet \cite{vinyals2016matching} & ResNet-18 & 52.91 $\pm$ 0.88 & 68.88 $\pm$ 0.69&-&-&72.36 $\pm$ 0.90&83.64 $\pm$ 0.60\\
ProtoNet \cite{snell2017prototypical} & ResNet-18 & 54.16 $\pm$ 0.82 & 73.68 $\pm$ 0.65&53.31 $\pm$ 0.89&72.69 $\pm$ 0.74&71.88 $\pm$ 0.91& 87.42 $\pm$ 0.48\\
Gidaris \cite{gidaris2018dynamic}  & ResNet-15 & 55.45 $\pm$ 0.89 & 70.13 $\pm$ 0.68&-&-&-&-\\
SNAIL \cite{mishra2018simple} & ResNet-15 & 55.71 $\pm$ 0.99 & 68.88 $\pm$ 0.92&-&-&-&-\\
TADAM \cite{oreshkin2018tadam}  & ResNet-15 & 58.50 $\pm$ 0.30 & 76.70 $\pm$ 0.30&-&-&-&-\\
Transductive \cite{dhillon2019baseline}$^{\dag}$ & ResNet-12 & 62.35 $\pm$ 0.66 & 74.53 $\pm$ 0.54&-&-&-&-\\
MetaoptNet \cite{lee2019meta}  & ResNet-18 & 62.64 $\pm$ 0.61 & 78.63 $\pm$ 0.46&65.99 $\pm$ 0.72 & 81.56 $\pm$ 0.53&-&-\\
TPN \cite{liu2018learning}$^{\dag}$& ResNet-12 & 53.75 $\pm$ 0.86 & 69.43 $\pm$ 0.67 & 57.53 $\pm$ 0.96 & 72.85 $\pm$ 0.74 & - & - \\
TEAM \cite{qiao2019transductive}$^{\dag}$ & ResNet-18 & 60.07 $\pm$ 0.59 & 75.90 $\pm$ 0.38 & - & - & 80.16 $\pm$ 0.52 & 87.17 $\pm$ 0.39 \\
CAN+T \cite{hou2019cross}$^{\dag}$& ResNet-12 & 67.19 $\pm$ 0.55 & 80.64 $\pm$ 0.35& 73.21 $\pm$ 0.58 & 84.93 $\pm$ 0.38&-&-\\
$^*$SimpleShot \cite{wang2019simpleshot}$^{\dag}$& ResNet-18 & 62.86 $\pm$ 0.20 & 79.22 $\pm$ 0.14 & 69.71 $\pm$ 0.23 & 84.13 $\pm$ 0.17 & 72.86 $\pm$ 0.20 & 88.57 $\pm$ 0.11\\
$^*$LaplacianShot\cite{ziko2020laplacian}$^{\dag}$ & ResNet-18 & 70.46 $\pm$ 0.23 & 81.76 $\pm$ 0.14 & 76.90 $\pm$ 0.25 & 85.10 $\pm$ 0.17 & 82.92 $\pm$ 0.21 & 90.77 $\pm$ 0.11 \\
$^*$EPNet\cite{rodriguez2020embedding}$^{\dag}$ & ResNet-18 & 63.83 $\pm$ 0.20 & 77.98 $\pm$ 0.15 & 70.08 $\pm$ 0.23 & 82.11 $\pm$ 0.18 & 73.32 $\pm$ 0.21 & 87.55 $\pm$ 0.13\\
EASE+Soft K means\cite{zhu2022ease}$^{\dag}$ & ResNet-12 & 57.00 $\pm$ 0.26 & 75.07 $\pm$ 0.21 & 69.74 $\pm$ 0.31 & 85.17 $\pm$ 0.21 & 76.72 $\pm$ 0.27 & 90.04 $\pm$ 0.16\\
Diff-ResNet(ours)$^{\dag}$ & ResNet-18 & \textbf{71.11} $\pm$ 0.24 & \textbf{82.07} $\pm$ 0.14 & \textbf{77.98} $\pm$ 0.25 & \textbf{85.75} $\pm$ 0.17 & \textbf{84.20} $\pm$ 0.21 & \textbf{91.12} $\pm$ 0.10\\
\hline
Qiao \cite{qiao2018few}  & WRN & 59.60 $\pm$ 0.41 & 73.74 $\pm$ 0.19&-&-&-&-\\
LEO \cite{rusu2018meta} & WRN & 61.76 $\pm$ 0.08 &77.59 $\pm$ 0.12& 66.33 $\pm$ 0.05 & 81.44 $\pm$ 0.09&-&-\\
ProtoNet \cite{snell2017prototypical} & WRN & 62.60 $\pm$ 0.20 & 79.97 $\pm$ 0.14&-&-&-&-\\
CC+rot \cite{gidaris2019boosting}  & WRN & 62.93 $\pm$ 0.45 & 79.87 $\pm$ 0.33& 70.53 $\pm$ 0.51 & 84.98 $\pm$ 0.36&-&-\\
MatchingNet \cite{vinyals2016matching} & WRN & 64.03 $\pm$ 0.20 & 76.32 $\pm$ 0.16&-&-&-&-\\
FEAT \cite{ye2020fewshot} & WRN & 65.10 $\pm$ 0.20 & 81.11 $\pm$ 0.14& 70.41 $\pm$ 0.23 & 84.38 $\pm$ 0.16&-&-\\
Transductive \cite{dhillon2019baseline}$^{\dag}$ & WRN & 65.73 $\pm$ 0.68 & 78.40 $\pm$ 0.52& 73.34 $\pm$ 0.71 & 85.50 $\pm$ 0.50&-&-\\
BD-CSPN \cite{liu2019prototype}$^{\dag}$ & WRN & 70.31 $\pm$ 0.93 & 81.89 $\pm$ 0.60& 78.74 $\pm$ 0.95 & 86.92 $\pm$ 0.63&-&-\\
$^*$SimpleShot \cite{wang2019simpleshot}$^{\dag}$ & WRN & 65.20 $\pm$ 0.20 & 81.28 $\pm$ 0.14 & 71.49 $\pm$ 0.23 & 85.51 $\pm$ 0.16 & 78.62 $\pm$ 0.19 & 91.21 $\pm$ 0.10 \\
$^*$LaplacianShot\cite{ziko2020laplacian}$^{\dag}$ & WRN & 72.90 $\pm$ 0.23 & 83.47 $\pm$ 0.14 & 78.79 $\pm$ 0.25 & 86.46 $\pm$ 0.17 & \textbf{87.70} $\pm$ 0.18 & 92.73 $\pm$ 0.10\\
$^*$EPNet\cite{rodriguez2020embedding}$^{\dag}$ &  WRN & 67.09 $\pm$ 0.21 & 80.71 $\pm$ 0.14 & 73.20 $\pm$ 0.23 & 84.20 $\pm$ 0.17 & 80.88 $\pm$ 0.20 & 91.40 $\pm$ 0.11\\
PT+NCM\cite{hu2021leveraging}& WRN & 65.35 $\pm$ 0.20 & 83.87 $\pm$ 0.13 & 69.96 $\pm$ 0.22 & 86.45 $\pm$ 0.15 & 80.57 $\pm$ 0.20 & 91.15 $\pm$ 0.10\\
EASE+Soft K means\cite{zhu2022ease}$^{\dag}$ & WRN & 67.42 $\pm$ 0.27 & \textbf{84.45} $\pm$ 0.18 & 75.87 $\pm$ 0.29 & \textbf{85.17} $\pm$ 0.21 & 81.01 $\pm$ 0.26 & 91.44 $\pm$ 0.14\\
Diff-ResNet(ours)$^{\dag}$ & WRN & \textbf{73.47} $\pm$ 0.23 & 83.86 $\pm$ 0.14 & \textbf{79.74} $\pm$ 0.25 & \textbf{87.10} $\pm$ 0.16 & \textbf{87.74} $\pm$ 0.19 & \textbf{92.96} $\pm$ 0.09\\
\bottomrule
\end{tabular}
\end{small}
\end{center}
\end{table*}

At the end of the section, we want to emphasize that our Diff-ResNet can achieve state-of-the-art. For fair comparison, we adopt tricks used in LaplacianShot~\cite{ziko2020laplacian}, which we elaborate in Appendix E.3. We also remove the balanced class assumption in \cite{hu2021leveraging,zhu2022ease} and report the corresponding results, as our Diff-ResNet does not utilize such assumption. We randomly sample 10000 5-way-1-shot and 5-way-5-shot classification tasks and report the average accuracy and corresponding 95\% confidence interval in Table~\ref{table:fewshot}. The results of networks for comparison in Table~\ref{table:fewshot} are collected from \cite{chen2019closerfewshot,wang2019simpleshot,ziko2020laplacian}. In all datasets with various backbones, Diff-ResNet obtains the highest classification accuracy, except on 5-shot task on \textit{mini}ImageNet with WRN-28-10 as backbone. The improvement of Diff-ResNet on 5-shot tasks is not as significant as on 1-shot tasks, which is consistent with the observation in Table~\ref{table:ablation}. As for the relatively small performance increase compared to LaplacianShot\cite{ziko2020laplacian}, we have thoroughly investigated the difference between our method and theirs in tha ablation study~\ref{table:ablation}(Internal Convection-Diffusion v.s. Diffusion).

Additionally, we investigate the computational cost of diffusion mechanism. The implementation of diffusion layers is simply small-scale matrix multiplication, which is very efficient using GPU. We run 1000 classification tasks using Diff-ResNet with different number of diffusion layers $r$, and report the average computation time per task and accuracy in Figure~\ref{fig:acc_time}. The total diffusion strength is fixed as $r\gamma=2.0$, except when $r=1$ we choose $\gamma=1.0$ for stability. We use a single GeForce RTX 2080 Ti to collect the running time. As stated before in Figure~\ref{fig:acc_step}, with fixed diffusion strength, there is no need of stacking too many layers. In both subfigures, 2 diffusion layers can already achieve the best result, whereas the increase in time compared to no diffusion is approximately 25\%. Moreover, the time of 10 diffusion layers roughly doubles that of without diffusion, and 10 layers is enough to achieve desired diffusion strength in all few-shot tasks. 

However, compared to methods~\cite{ziko2020laplacian,wang2019simpleshot,rodriguez2020embedding}, our method requires more time. In each few-shot task, at the inference phase, our method need to train a tiny neural network. Notice that we are not trying to retrain or finetune the backbone network. Rather, we train a tiny network that only contains two hidden layers and a diffusion block. The acceleration of inference will be studied in future.

\begin{figure}[hbtp]
	\centering
	\subfloat[ResNet-18]{\includegraphics[width=0.5\linewidth]{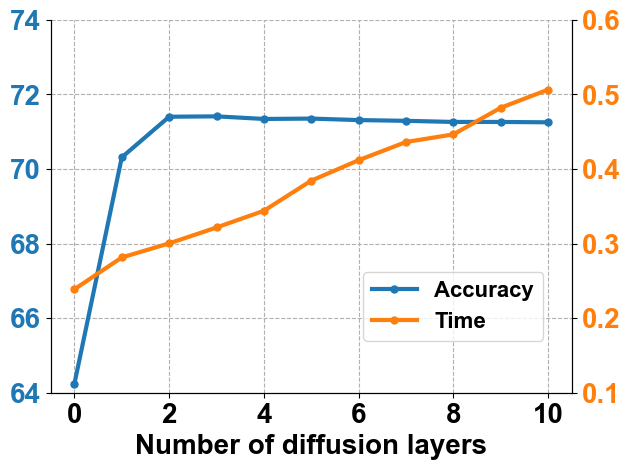}}
	\hfill
	\subfloat[WRN]{\includegraphics[width=0.5\linewidth]{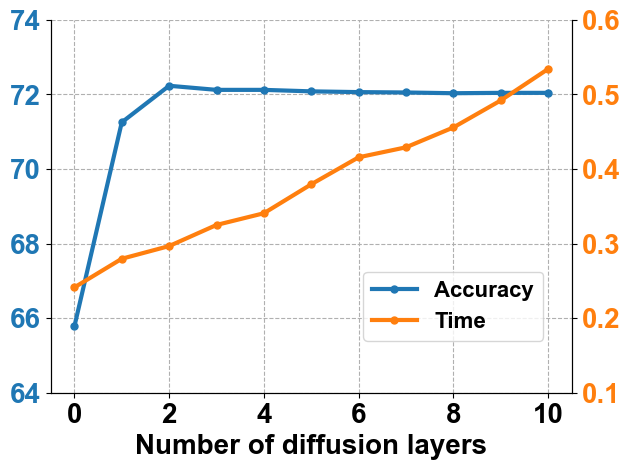}}
	
	\caption{The computation time of each task and accuracy on \textit{mini}ImageNet with ResNet-18 and WRN as backbone. The x-axis represents number of diffusion layers, 0 means no diffusion, right y-axis is the computation time (seconds) of each task, left y-axis is the average accuracy.}
	\label{fig:acc_time}
\end{figure}

%% file: data/conclusion.tex
In this paper, inspired by the ODE model with diffusion mechanism, we propose a novel Diff-ResNets by adding a simple yet powerful diffusion layer to the residual blocks. We conduct theoretical analysis of the diffusion mechanism and prove that the diffusion term will significantly increase the ratio between local intra-class distance and inter-class distance. The performance of proposed Diff-ResNets is verified by extensive experiments on few-shot learning and semi-supervised graph learning problems. Our future work involves the robustness of diffusion mechanism, the acceleration of Diff-ResNet in inference phase of few-shot learning, the extension from diffusive ODE to diffusive PDE, and the effect of diffusion in semi-supervised learning with extremyly low label rate.

%% file: data/biography.tex
\begin{IEEEbiography}[{\includegraphics[width=1in,height=1.2in,clip,keepaspectratio]{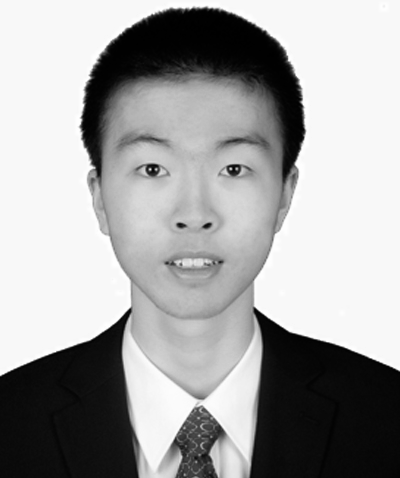}}]{Tangjun Wang} is a Ph.D. student in Department of Mathematical Sciences, Tsinghua University supervised by Zuoqiang Shi. His research interests include semi-supervised learning and applications of partial differential equation in neural networks.
\end{IEEEbiography}

\begin{IEEEbiography}[{\includegraphics[width=1in,height=1.2in,clip,keepaspectratio]{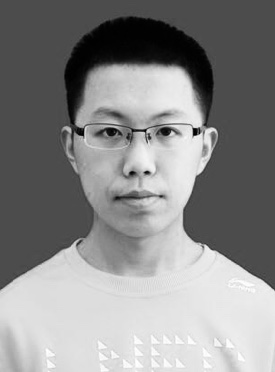}}]{Zehao Dou} is a Ph.D. student in Department of Statistics and Data Science, Yale University supervised by Harry Zhou and John Lafferty. His research interests include statistics, optimization, machine learning theory and reinforcement learning theory. 
\end{IEEEbiography}

\begin{IEEEbiography}[{\includegraphics[width=1in,height=1.2in,clip,keepaspectratio]{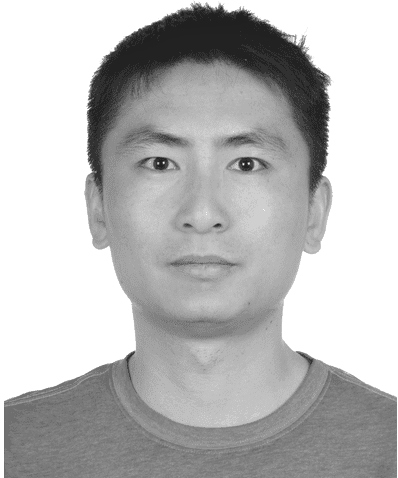}}]{Chenglong Bao} is an assistant professor in Yau mathematical sciences center, Tsinghua University and Yanqi Lake Beijing Institute of Mathematical Sciences and Applications. He received his Ph.D. from department of mathematics, National University of Singapore in 2014. His main research interests include mathematical image processing, large scale optimization and its applications.
\end{IEEEbiography}

\begin{IEEEbiography}[{\includegraphics[width=1in,height=1.2in,clip,keepaspectratio]{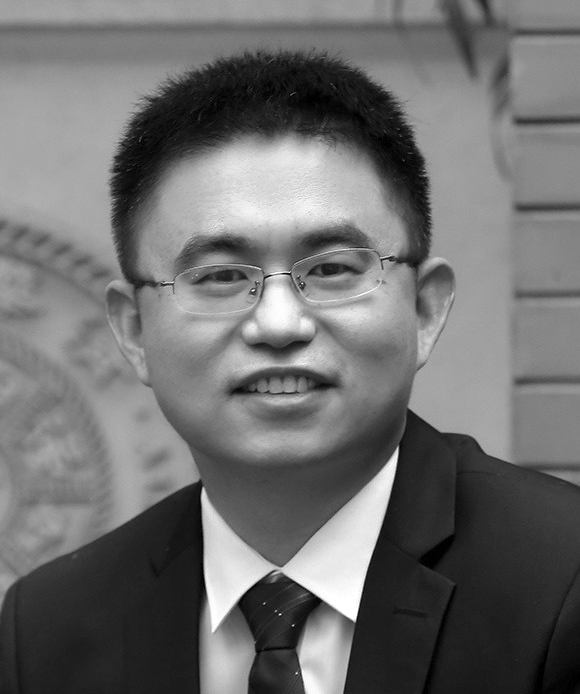}}]{Zuoqiang Shi} is an associate professor in Department of Mathematical Sciences, Tsinghua University. He received his Ph.D. from Zhou Pei-Yuan center for Applied Mathematics, Tsinghua University in 2008. His main research interests include numerical methods of partial differential equations and its applications, mathematical image processing.
\end{IEEEbiography}

\vfill

%% file: data/appendix_stability.tex
\section{Proof of Stability Condition}
\label{section:stability}
In this section, we will give the stability condition of discretization of the diffusion step. When the convection term equals zero, the forward Euler discretization of \eqref{eq:diffusion_step} is 
\[x_i^{k+1} = x_i^k - \gamma\sum_{j=1}^N w_{ij}(x_i^k-x_j^k),~i=1,2\ldots,N.\]
$\Delta t $ is again absorbed in $\gamma$. Define $X^k=[x_1^k,\ldots,x_N^k]$, the vectorized update scheme is
\begin{equation}
	\label{eq:update}
	X^{k+1} = X^k - \gamma(\Lambda-W) X^k,
\end{equation}
where $\Lambda = \mathrm{diag}(d_i)$ with $d_i=\sum_{j=1}^Nw_{ij}$ for all $i=1,2,\ldots,N$. We demand all the $\{d_i\}$ share the same value $d_1=d_2=\cdots=d_N:=d$. The next proposition shows the stability condition of iteration \eqref{eq:update}.
\begin{Proposition}
	\label{prop:sta}
	If $\gamma\in\left(0,\frac{1}{d}\right]$, the iteration \eqref{eq:update} is a contraction.
\end{Proposition}

\begin{proof}
Matrix with spectral radius smaller than 1 is a contraction matrix, since
\[\|Ax\|_2\leqslant\|A\|_2\|x\|_2\leqslant\|x\|_2\]
for any $x$ if $\|A\|_2=\rho(A)\leqslant1$.

Thus, we are trying to prove that $\rho(I-\gamma(\Lambda-W))\leqslant 1$. In fact, we will show that $\rho(I-\gamma(\Lambda-W))=1$. Denote $A=I-\gamma(\Lambda-W)$

First, it is obvious that 1 is an eigenvalue of $A$. $\Lambda-W$ is a matrix with row sum 0 by construction of $\Lambda$, thus 0 is an eigenvalue of $\Lambda-W$ with $\boldsymbol{1}$ its correspoding eigenvector. Therefore, 1 is an eigenvalue of $A=I-\gamma(\Lambda-W)$.

Then, as $A$ is obviously a symmetric matrix, all of its eigenvalues $\lambda$ are real. We will show that all the eigenvalues of $A$ lies in $[-1,1]$ using the Gershgorin disk theorem. In the $i$-th row of $A$, the diagonal entry is $1-\gamma(d-w_{ii})$, and the off-diagonal entries are $\gamma w_{ij}$ for all $j\neq i$. Thus, the disk decided by this row is centered at $1-\gamma(d-w_{ii})$, with radius $\sum_{j\neq i}\gamma w_{ij}$. Remind again that $d=\sum_{j}w_{ij}$, so $\sum_{j\neq i}\gamma w_{ij}=\gamma(d-w_{ii})$.  Using $\gamma\in\left(0,\frac{1}{d}\right)$, we have
\[
\begin{aligned}
	-1&\leqslant 1-2\gamma d\\
	&\leqslant 1-2\gamma d+2\gamma w_{ii}\\
	&=1-\gamma (d-w_{ii})-\gamma (d-w_{ii})\\
	&\leqslant \lambda\\
	&\leqslant 1-\gamma (d-w_{ii})+\gamma (d-w_{ii})=1\\
\end{aligned}
\]
The second inequality is because all entries in $W$ are non-negative. In conclusion, the spectral radius of $A$ is 1.
\end{proof}
In the algorithm, weight matrix is symmetrically normalized, thus $d=O(1)$. Therefore, in our batch diffusion mechanism, once the step size meets a relatively loose constraint, the stability can be guaranteed.

%% file: data/appendix_proof1.tex
\section{Proof of Theorem 1}
We write the weight of the second layer as $a_{t}^{(i)}=\lambda_{t}^{(i)}\cdot\beta_{t}^{(i)}$ for proof convenience. $\lambda_{t}^{(i)}\in \mathbb{R}, \beta_{t}^{(i)}\in \mathbb{R}^d$. Then
\[f(x(t),\theta(t))=\sum_{i=1}^{w}\lambda_{t}^{(i)}\sigma(\w_{t}^{(i)}\cdot x(t)+b_{t}^{(i)})\beta_{t}^{(i)}\]

First of all, let us deal with the simplest situation. Each region $S_{i,j},~(i,j)\in [k]\times[l]$ only has 1 point, and the width of the 2-layer network is also 1.
\begin{Lemma}
	\label{lemma:1}
	Assume: \[f(x(t),\theta(t))=\lambda_{t}\sigma(\w_{t}\cdot x(t)+b_{t})\beta_{t}\]

	Then, given the dataset $\{(x_i,y_i)\}_{i=1}^N$, we can construct the function $f$ above with $2N+O(d)$ different variables and $N$ layers, so that the final-step features of these data points are linear separable.
\end{Lemma}
\begin{proof}
Given $N$ data points $x_1,x_2,\cdots, x_N\in R^{d}$ with their corresponding labels.

\vbox{}

1) There exists $\w\in R^{d}$, such that $\w\cdot x_{i},~i\in [N]$ are all distinct. This is obvious since $\mathbf{W}_{ij}=\{\w|\w\cdot x_{i}=\w\cdot x_{j}\}$ forms a hyperplane of $R^{d}$, which has 0 measure. Therefore their finite union
\[K = \mathop{\bigcup}\limits_{1\leq i<j\leq N} \mathbf{W}_{ij}\]
also has 0 measure. We only need to pick a $\w^*\in R^{d}\setminus K$, so that this $\w^*$ meets our need. Let $\w_{t}$ be a constant value function and $\w_{t}\equiv \w^*$ holds for all $t\in [0,1]$. 
Since the dot products $\w^*\cdot x_{i}$ are pairwise distinct, we can assume:
\[\w^*\cdot x_{1}<\w^*\cdot x_{2}<\cdots<\w^*\cdot x_N\]
Denote $A_{i}=\w^*\cdot x_{i}$ and then $A_{1}<A_{2}<\cdots<A_N$

\vbox{}

2) Without loss of generality, we assume the $d$-th component $\w^*_{d}\neq 0$, then we denote:
\[\beta^*=\left(1,1,\cdots,1,-\frac{\w^*_{1}+\w^*_{2}+\cdots+\w^*_{d-1}}{\w^*_{d}}\right)\]
Let $\beta_{t}$ also be a constant value function and $\beta_{t}\equiv \beta^*$. Then we have: $\w^*\perp\beta^*$ and our main function $f$ becomes:
\[x'(t)=f(x(t),\theta(t))=\lambda_{t}\sigma(\w^*\cdot x(t)+b_{t})\beta^*\]
We notice that:
\begin{align*}
(\w^*\cdot x(t))'&=\w^*\cdot(\lambda_{t}\sigma(\w^*\cdot x(t)+b_{t})\beta^*)\\
&=\lambda_{t}\sigma(\w^*\cdot x(t)+b_{t})\cdot (\w^*\cdot\beta^*)=0
\end{align*}
which means during the flow, $\w^*\cdot x(t)$ remains constant. Thus
\[\w^*\cdot x_{i}(t)\equiv \w^*\cdot x_{i}(0)=\w^*\cdot x_{i}=A_{i}~~\forall t\in [0,1],i\in[N]\]
and $x_{i}'(t)=\lambda_{t}\sigma(A_{i}+b_{t})\beta^*$. Using the definition of ReLU activation function , when $A_{i}+b_{t}<0$, $x_{i}(t)$ remains unchanged.

\vbox{}

3) Pick $N$ real numbers $B_{1},B_{2},\cdots,B_N$ such that:
\[B_{1}<A_{1}<B_{2}<A_{2}<\cdots<B_N<A_N\]
Now we construct the time-dependent scalar $b_{t}$:
\[b_{t}=-B_{i}, ~\forall t\in\left[\frac{i-1}{N},\frac{i}{N}\right), i\in [N]\]
As we can see, $b_{t}$ is piecewise constant with $N$ pieces, or so-called layers.

\vbox{}

4) Finally, we construct suitable time-dependent scalar $\lambda_{t}$ to make the final-step features $x_{i}(1),~i\in [N]$ linear separable. Pick 2 real numbers $C_{1}<C_{2}$. We will choose a suitable position for $x_{i}(1)$ one by one. To be concrete, we will make the first component of $x_{i}(1)$:
\[(x_{i}(1))_{1}=C_{y_{i}}\]
Here $y_{i}\in \{1,2\}$ is the corresponding label of data point $x_{i}$.

Notice that, for each $j\in [N-1]$, when $t>\frac{j}{N}$, $x_{1}(t),x_{2}(t),\cdots,x_{j}(t)$ remains constant, because according to our construction of $b_{t}$ and sequence $\{B_{i}\}$, $A_{i}+b_{t}\leq A_{j}-B_{j+1}<0$ when $t>\frac{j}{N}$ and $ i\leq j$.

In the first time step, $t\in[0,\frac{1}{N})$, data point $x_{1}(0)=x_{1}$ has its corresponding label $y_{1}\in \{1,2\}$. Since $b_{t}=-B_{1}$
\[x_{1}'(t)=\lambda_{t}(A_{1}-B_{1})\beta^*\]
Make
\[\lambda_{t}\equiv\lambda_{1}=\frac{C_{y_{1}}-(x_{1}(0))_{1}}{A_{1}-B_{1}}\cdot N\quad\forall t\in\left[0,\frac{1}{N}\right)\]
Then
\begin{align*}
	x_{1}(1)&=x_{1}(\frac{1}{N})\\
	&=x_{1}(0)+\frac{1}{N}\lambda_{1}(A_{1}-B_{1})\beta^*\\
	&=x_{1}(0)+(C_{y_{1}}-(x_{1}(0))_{1})\beta^*
\end{align*}
Note that the first component of $\beta^*$ is 1, so $(x_{1}(1))_{1}=C_{y_{1}}$. 

Similarly, in the $j$-th time step, $t\in \left[\frac{j-1}{N},\frac{j}{N}\right)$ and $b_{t}=-B_{j}$. We choose a suitable position for $x_{j}(1)$. Make
\[\lambda_{t}\equiv \lambda_{j}=\frac{C_{y_{j}}-(x_{j}(\frac{j-1}{N}))_{1}}{A_{j}-B_{j}}\cdot N\quad\forall t\in\left[\frac{j-1}{N},\frac{j}{N}\right)\]
so that:
\begin{align*}
	x_{j}(1)&=x_{j}\left(\frac{j}{N}\right)\\
	&=x_{j}\left(\frac{j-1}{N}\right)+\frac{1}{N}\lambda_{j}(A_{j}-B_{j})\beta^*\\
	&=x_{j}\left(\frac{j-1}{N}\right)+(C_{y_{j}}-\left(x_{j}\left(\frac{j-1}{N}\right)\right)_{1})\beta^*
\end{align*}
and then its first component $(x_{j}(1))_{1}=C_{y_{j}}$.

To sum up, $\lambda_{t}$ is piecewise constant with $N$ pieces. $\lambda_{t}=\lambda_{j}$ when $t\in[\frac{j-1}{N},\frac{j}{N})$. Here:
\[\lambda_{j}=\frac{C_{y_{j}}-(x_{j}(\frac{j-1}{N}))_{1}}{A_{j}-B_{j}}\cdot N\qquad\forall t\in\left[\frac{j-1}{N},\frac{j}{N}\right)\]
After all these time steps, we can guarantee that for each $i\in[N]$,
\[(x_{i}(1))_{1}=C_{y_{i}}\]
In other words, final-step features with the same corresponding label $y_i\in \{1,2\}$ are on the same hyperplane, vectors with fixed first component $\{\mathbf{v}:(\mathbf{v})_{1}=C_{y_i}\}$. Therefore, it is obvious that they can be easily separated by a hyperplane
\[\Big\{\mathbf{v}:(\mathbf{v})_{1}=\frac{C_{1}+C_{2}}{2}\Big\}\]
which meets our satisfaction. In this construction of function $f$, $\w_{t},\beta_{t}$ remains constant, and $\lambda_{t},b_{t}$ changes every time step. In all, there are $2N+O(d)$ variables.
\end{proof}

\vbox{}

In the setting above, the width of this network is 1. Now we will deal with a slightly harder situation: network with width $w$. And our next lemma can perfectly answer this question: it's possible to use fewer layers (but same number of parameters) to make the final features linear separable when using wider network. 

\begin{Lemma}
	\label{lemma:2}
	Assume: \[f(x(t),\theta(t))=\sum_{i=1}^{w}\lambda_{t}^{(i)}\sigma(\w_{t}^{(i)}\cdot x(t)+b_{t}^{(i)})\beta_{t}^{(i)}\]
	Here $w$ is the width of the network.
	
	Then, given the dataset $\{(x_i,y_i)\}_{i=1}^N$, we can construct the function $f$ above with $2N+O(d)$ different variables and $\lceil N/w \rceil$ layers, so that the final-step features of these data points are linear separable.
\end{Lemma}

\begin{proof}
Exactly like Lemma \ref{lemma:1}, we let $\w_{t}^{(i)}\equiv \w^*$ and $\beta_{t}^{(i)}\equiv\beta^*$ for $\forall i\in[w]$. And we will construct suitable time-dependent scalars $\lambda_{t}^{i},b_{t}^{i}$ in order to make:
\[(x_{i}(1))_{1}=C_{y_{i}}\]
holds for all $i\in [N]$.

Define $B_{N+1}=B_{N+2}=\cdots=A_{N}+1$, and denote $L=\left\lceil\frac{N}{w}\right\rceil$. Before further analysis, we split the whole time period $[0,1]$ into $L$ equal time steps.

In the $j$-th time step, $t\in[\frac{j-1}{L},\frac{j}{L})$, let:
\[b_{t}^{(i)}\equiv-B_{(j-1)w+i}, ~\lambda_{t}^{(i)}\equiv \lambda_{(j-1)w+i}~~\forall i\in [w]\]
where $\lambda_{i}$ is undetermined.

Notice that each $b_{t}^{(i)}$ and $\lambda_{t}^{(i)}$ is piecewise constant with $L$ pieces. Similar to Lemma \ref{lemma:1}, after the $j$-th time step, $x_{1}(t),x_{2}(t),\cdots,x_{jw}(t)$ remains unchanged, because $\forall i\in[jw],t>\frac{j}{L}$, $A_{i}+b_{t}^{(j)}\leqslant A_{jw}-B_{jw+1}<0$. Therefore, our plan is to put $x_{(j-1)w+1}(1),x_{(j-1)w+2}(1),\cdots,x_{jw}(1)$ into suitable positions during the $j$-th time step, which means:
\begin{align*}
&C_{y_{i}}=\left(x_{i}\left(\frac{j-1}{L}\right)\right)_{1}+\frac{1}{L}\sum_{k=1}^{w}\lambda_{(j-1)w+k}\cdot\sigma(A_{i}-B_{(j-1)w+k})\\
&=\left(x_{i}\left(\frac{j-1}{L}\right)\right)_{1}+\frac{1}{L}\sum_{k=1}^{i-(j-1)w}\lambda_{(j-1)w+k}\cdot(A_{i}-B_{(j-1)w+k})\\
&\quad \qquad \qquad \qquad \qquad \qquad \qquad \qquad \qquad \forall (j-1)w<i\leqslant jw
\end{align*}

It is a linear system of equations, and we can solve these $\lambda_{i},~(j-1)w<i\leqslant jw$ through the linear functions above uniquely.

After all of these time steps, we can guarantee that:
\[(x_{i}(1))_{1}=C_{y_{i}}\]
holds for all $i\in[N]$. 

In this proceedings, we use only $L = \left\lceil\frac{N}{w}\right\rceil$ layers to meet our satisfaction. However, the number of variables we use is also $2N+O(d)$, which does not change with the increase of network width.
\end{proof}

\vbox{}

At last, we can deal with the original Theorem \ref{thm:1}, which has width $w$ and each region has several points. It is a simple extension of Lemma \ref{lemma:2}.
\begin{proof}
	For simplicity, we reorganize $M=kl$ subsets and renumber them as $\Gamma_{m}=S_{i,j},~m\in [M],~(i,j)\in [k]\times[l]$, as we do not have to distinguish whether regions are from the same class or different classes. Since they can be separated by a set of $M-1$ parallel hyperplanes, there exists a vector $\w^*\in \mathbb{R}^{d}$, such that the following intervals do not intersect with each other:
	\[\w^*\cdot \Gamma_{i}\triangleq\{\w^*\cdot x_{i}|x_{i}\in\Gamma_{i}\}~~i\in[M]\]
	We can assume $\forall x_{1}\in \Gamma_{1}, x_{2}\in \Gamma_{2},\cdots,x_{M}\in \Gamma_{M}$:
	\[\w^*\cdot x_{1}<\w^*\cdot x_{2}<\cdots<\w^*\cdot x_{M}\]	
	Therefore, there exists real numbers $B_{1}<A_{1}\leqslant B_{2}<A_{2}\leqslant\cdots\leqslant B_{M}<A_{M}$ such that $\w^*\cdot \Gamma_{i}\subseteq [B_{i},A_{i}]$. Just like the proof of Lemma \ref{lemma:2}, we let $\w_{t}^{(i)}\equiv \w^*$ and $\beta_{t}^{(i)}\equiv\beta^*$ for $\forall i\in[w]$. Here, without lack of generality, assume the $d$-th component of $\w^*$ is non-zero, and let
	\[\beta^*=\left(1,1,\cdots,1,-\frac{\w^*_{1}+\w^*_{2}+\cdots+\w^*_{d-1}}{\w^*_{d}}\right)\]
	
	After that, we can treat these $M$ regions just as $M$ data points. The only difference is: instead of making each final-step feature with the same corresponding label has the same first component, we make them in the same interval. 
	
	Pick 2 intervals $\C_{y_i}=[a_{y_i},b_{y_i}]$ such that: (1) $a_{1}<b_{1}<a_{2}<b_{2}$, therefore these two intervals don't intersect with each other. (2) $b_{y_i}-a_{y_i} > D$ holds for each $y_i\in\{1,2\}$. So that, each interval can hold a complete region $\Gamma_{i}$. Then we can use exactly the same way as Lemma \ref{lemma:2} to make $x_i(1)\subseteq \C_{y_i}$, and we only have to change the two real numbers $C_{1}<C_{2}$ into the two intervals above. In the end, we can separate the final-step feature regions with the following parallel hyperplane:
	\[\Big\{\mathbf{v}:(\mathbf{v})_{1}=\frac{d_{1}+c_{2}}{2}\Big\}\]
	
	Similar to Lemma \ref{lemma:2}, we use $L = \left\lceil\frac{M}{w}\right\rceil$ layers. Moreover, the number of variables we use decreases to $2M+O(d)$, which is a significant change.
\end{proof}

%% file: data/appendix_proof2.tex
\section{Proof of Theorem 2}
\begin{proof}
Similar to the last subsection, we reorganize $M=kl$ subsets and renumber them as $\Gamma_{m}=S_{i,j},~m\in [M],~(i,j)\in [k]\times[l]$. According to the definition of upper bound of diameters $D$, assume $\Gamma_{i}\subseteq B(O_{i},D/2)$ and denote $R = D/2$. 

First we introduce some notations. Denote hypersphere in $d$-dimension with radius r as $S^{d-1}(r)$. If $r=1$, which is a unit hypersphere, we simply write $S^{d-1}$.  Denote $A(\omega)$ as the surface area of $\omega$. $\Gamma(x)$ is the Gamma Function.

The goal to find a unit normal vector $\w\in S^{d-1}$, such that $\forall~b\in R$, hyperplane $\w\cdot\x+b=0$ does not intersect with any two regions $\Gamma_{i}$ and $\Gamma_{j}$.

Denote:
\begin{align*}
	K_{ij}=\{\w:&\|\w\|_{2}=1,~\exists~b\in \mathbb{R}, s.t.~\text{hyperplane}~\w\x+b=0 \\
	&\text{intersects with both}~\Gamma_{i}~\text{and}~\Gamma_{j}\}
\end{align*}

Then we need to prove:
\begin{equation}
\label{eq:kij}
\bigcup\limits_{1\leqslant i<j\leqslant M}K_{ij}\subsetneqq S^{d-1}
\end{equation}
so that every unit normal vector $\w$ which doesn't belong to any $K_{ij}$ meets our satisfaction. In order to prove (\ref{eq:kij}), we compare the surface area of the two sides.

It is well known that on the right side
\[A(S^{d-1}) = \frac{2\pi^{d/2}}{\Gamma(\frac{d}{2})}\]

We are going to calculate the surface area of the left side. For any $i\neq j\in[M]$, we do the scaling as follows:
\begin{align*}
K_{ij} &=\{\w:\|\w\|_{2}=1,~\exists~b\in \mathbb{R}, s.t.~\text{hyperplane}~\w\x+b=0\\
&\qquad\text{intersects with both}~\Gamma_{i}~\text{and}~\Gamma_{j}\}\\
&\subseteq\{\w:\|\w\|_{2}=1,~\exists~b\in \mathbb{R}, s.t.~\text{distances from }~O_{i},O_{j}\\
&\qquad\text{to}~\w\x+b=0~\text{are at most}~R\}\\
&= \{\w:\|\w\|_{2}=1,~\exists~b\in \mathbb{R}, s.t.~|\w\cdot\OO_{i}+b|\leqslant R,\\
&\qquad|\w\cdot\OO_{j}+b|\leqslant R\}\\
&= \{\w:\|\w\|_{2}=1,~|\w\cdot\OO_{i}-\w\cdot\OO_{j}|\leqslant 2R\}\\
&\subseteq\{\w:\|\w\|_{2}=1,~|\w\cdot\hat{\overrightarrow{O_{i}O_{j}}}|\leqslant \frac{2R}{L}=\frac{D}{L}\}
\end{align*}

Here, $\hat{\overrightarrow{O_{i}O_{j}}}$ means the unit vector in the direction of $\overrightarrow{O_{i}O_{j}}$ and the last step above is because $\|O_{i}O_{j}\|\geqslant dist(\Gamma_{i},\Gamma_{j}) > L$. Denote $t = \frac{D}{L}, \e_{d}=(0,0,\cdots,0,1)$. By our assumption obviously $t<1$. Next we will calculate the surface area of $K_{ij}$.
\[A(K_{ij})\leqslant\int_{\w\in S^{d-1},|\w\cdot\hat{\overrightarrow{O_{i}O_{j}}}|\leqslant t}dS=\int_{\w\in S^{d-1},|\w\cdot \e_{d}|\leqslant t}dS\]

It is the surface area of a \emph{hyperspherical segment}: the solid defined by cutting a hypersphere with a pair of parallel planes $\{\mathbf{v}:(\mathbf{v})_d=t\}$ and $\{\mathbf{v}:(\mathbf{v})_d=-t\}$. It can also be seen as a complete sphere excluding upper and lower \emph{hyperspherical caps}. The area of a hypersherical cap in a $d$-dimensional sphere of radius $r$ can be obtained by integrating the surface area of an $(d-1)$-dimensional sphere of radius $rsin(\theta)$ with arc element $rd\theta$ over a great circle arc~\cite{li2011concise}. Here $r=1$, and $\theta$ is integrated over 0 to $\varphi=arccos(t)$, which is the colatitude angle, i.e., the angle between a vector of the sphere and its $d^{\text{th}}$ positive axis.
\begin{align*}
A(K_{ij})&\leqslant A(S^{d-1})-2\int_{0}^{\arccos(t)}A(S^{d-2}(\sin\theta))d\theta\\
&=A(S^{d-1})-2~\frac{2\pi^{(d-1)/2}}{\Gamma(\frac{d-1}{2})}\int_{0}^{\arccos(t)}\sin^{d-2}\theta d\theta
\end{align*}

It is obvious that when $t=0$, the hyperspherical segment is just the whole sphere, which means
\[A(S^{d-1})= 2~\frac{2\pi^{(d-1)/2}}{\Gamma(\frac{d-1}{2})}\int_{0}^{\frac{\pi}{2}}\sin^{d-2}\theta d\theta\]
Thus
\begin{align*}
A(K_{ij})&\leqslant 2~\frac{2\pi^{(d-1)/2}}{\Gamma(\frac{d-1}{2})}\int_{\arccos(t)}^{\frac{\pi}{2}}\sin^{d-2}\theta d\theta\\
&=A(S^{d-1})\frac{2\Gamma(\frac{d}{2})}{\sqrt{\pi}\Gamma(\frac{d-1}{2})}\int_{\arccos(t)}^{\frac{\pi}{2}}\sin^{d-2}\theta d\theta
\end{align*}
Therefore:
\begin{align*}
\frac{A(K_{ij})}{A(S^{d-1})}&\leqslant \frac{2\Gamma(\frac{d}{2})}{\sqrt{\pi}\Gamma(\frac{d-1}{2})}\int_{\arccos(t)}^{\frac{\pi}{2}}\sin^{d-2}\theta d\theta\\
&\leqslant\frac{2\Gamma(\frac{d}{2})}{\sqrt{\pi}\Gamma(\frac{d-1}{2})}(\pi/2-\arccos(t))\\
&=\frac{2\Gamma(\frac{d}{2})}{\sqrt{\pi}\Gamma(\frac{d-1}{2})}\arcsin(t)
\end{align*}

Next we will estimate its upper bound. From the graph of function, we can obtain the upper bound of $\arcsin(t)\le \frac{\pi}{2}t, ~\forall t\in[0,1]$. The upper bound of the other part containing Gamma function is given in the following lemma.

\begin{Lemma}
	\[\frac{\Gamma(\frac{d}{2})}{\Gamma(\frac{d-1}{2})}< \frac{d}{2}\quad \forall d\in \mathbb{N}\]
\end{Lemma}

\begin{proof}
	When $d\leqslant5$, the lemma can be easily verified using exact values of the Gamma function. Suppose $d\geqslant6$. Note that Gamma function $\Gamma(x)$ is monotonically increasing when $x\ge 2$
	
	If d is odd, then $\frac{d-1}{2}$ is an integer, and $\frac{d}{2}>2$
	\[\frac{\Gamma(\frac{d}{2})}{\Gamma(\frac{d-1}{2})}\leqslant\frac{\Gamma(\frac{d+1}{2})}{\Gamma(\frac{d-1}{2})}=\frac{d-1}{2}<\frac{d}{2}\]
	
	In the equation above, we use the property of Gamma function: $\Gamma(x+1)=x\Gamma(x)$ if $x$ is an integer.
	
	Similarly, if d is even, then $\frac{d}{2}$ is an integer, and $\frac{d}{2}-1\geqslant2$
	\[\frac{\Gamma(\frac{d}{2})}{\Gamma(\frac{d-1}{2})}\leqslant\frac{\Gamma(\frac{d}{2})}{\Gamma(\frac{d}{2}-1)}=\frac{d}{2}-1<\frac{d}{2} \]
\end{proof}

Therefore:
\[\frac{A(K_{ij})}{A(S^{d-1})}<\frac{2}{\sqrt{\pi}}\frac{d}{2}\frac{\pi}{2}t=\frac{\sqrt{\pi}d}{2}t\]

Finally, we are able to calculate the surface area of the left side of (\ref{eq:kij}):
\begin{align*}
A\left(\bigcup\limits_{1\leqslant i<j\leqslant M}K_{ij}\right) &\leqslant \sum_{1\leqslant i<j\leqslant M}A(K_{ij})\\
&<\frac{\sqrt{\pi}d}{2}t A(S^{d-1})\cdot\binom{M}{2}\\
&\leqslant A(S^{d-1})
\end{align*}
Here, we use the assumption that:
\[t = \frac{D}{L}\leqslant \left(\frac{M(M-1)\sqrt{\pi}}{4}d\right)^{-1}\]
Thus, the correctness of (\ref{eq:kij}) is obvious.
\end{proof}

%% file: data/appendix_proof3.tex
\section{Proof of Proposition 1}
\begin{proof}
	Following proof of Theorem 1 and 2, renumber the $M=kl$ subsets and denote them $\Gamma_{m} = S_{i,j},~m\in [M],~(i,j)\in [k]\times[l]$ for notation convenience. We consider the forward Euler discretization of the diffusion process:
	\[x_i^{k+1} = x_i^k - \gamma\sum_{j=1}^N w_{ij}(x_i^k-x_j^k),~\forall i\in [N]\]
	
	We will first prove that with each update scheme, the new region $\Gamma_i^{k+1} \subseteq \Gamma_{i}^k$, thus $L(t)$ in monotonic non-decreasing. For a specific data point $x_{i}$, suppose $x_{i}\in \Gamma_{m}$. Since $\sum_{j=1}^N w_{ij}=1$ by normalization, the diffusion is:
	\[x_i^{k+1} =(1-\gamma)x_i^k+\gamma\sum_{j=1}^N w_{ij} x_j^k\]
	
	By assumption, points in each subset $\Gamma_{m}$ forms a connected component in graph $G$. Thus, the nearest $n_{\mathrm{top}}$ points of $x_i$ are all from $\Gamma_{m}$, i.e. $w_{ij}>0$ only if $x_j \in \Gamma_{m}$ \footnote{Reverse not necessarily true. $x_j \in \Gamma_{m}$ does not necessarily imply $w_{ij}>0$.}. Use the condition $\Gamma_{m}$ is convex, and $\sum_{j=1}^N w_{ij}=1$ again, the weighted sum $\sum_{j=1}^N w_{ij}x_j$ also lies in the convex region $\Gamma_{m}$. Finally, as $(1-\gamma)+\gamma=1$, we get $x_i^{k+1} \in \Gamma_{m}$. In other words, after batch diffusion, the new convex region $\Gamma_i^{k+1} \subseteq \Gamma_{i}^k$.
	
	Then we will show the upper bound of diameters $D(t)$ decreases exponentially to 0 with $t$. We may write diffusion mechanism in the vectorized form:
	\begin{equation}
		\frac{dX(t)}{dt}+\gamma (\Lambda-W)X(t)=0, \quad X(0)=X.
		\label{eq:continuous}
	\end{equation}
	where $X(t)=[x_1(t),\ldots,x_N(t)]$, $X=[x_1,\ldots,x_N]$, $\Lambda = \mathrm{diag}(d_i)$ with $d_i=\sum_{j=1}^Nw_{ij}$ for all $i=1,2,\ldots,N$, $W(i,j)=w_{ij}$. $L = \Lambda-W$ is called graph Laplacian\footnote{Abuse of notation with the lower bound of distances $L$}. Denote the eigenvalues and corresponding eigenvectors of $L$ as $\lambda_i$ and $\boldsymbol{v}_i,~ i\in [N]$. Since $L$ is a positive semi-definite symmetric matrix, $\lambda_i$ are real and non-negative. Suppose $0\leqslant\lambda_1\leqslant\lambda_2\leqslant\cdots\leqslant\lambda_N$. In spectral clustering literatures~\cite{chung1997spectral}, the reliance of graph diffusion on $L$ is well studied. A well-known result is: the multiplicity of 0 eigenvalue of the Laplacian equals the number of connected components of graph G. 
	
	We start with the simple case, where the graph $G$ only admits one connected component, i.e. when all data points belong to the same subclass. In this case, $\lambda_1=0$ with $\boldsymbol{v}_1= \boldsymbol{1}$ its correspoding eigenvector, and $\lambda_i>0,~\forall i\neq 1$. We will prove all points converge to their central. The spectral solution of \eqref{eq:continuous} is~\cite{chung1997spectral}:
	\begin{equation}
		X(t)=\sum_{i=1}^N\frac{X(0)\cdot \boldsymbol{v}_i}{|\boldsymbol{v}_i|^2}e^{-\gamma \lambda_{i} t}\boldsymbol{v}_i
	\end{equation}
	
	As $t\rightarrow\infty$, $e^{-\gamma \lambda_{i} t}\rightarrow 0$ if $\lambda_{i}>0$. Define 
	\[m_c= \frac{x_1+x_2+\cdots +x_N}{N}\] as the central point of data points. Then
	\begin{align*}
	\lim_{t\rightarrow \infty} X(t) &= \lim_{t\rightarrow \infty} \sum_{i=1}^N\frac{X(0)\cdot \boldsymbol{v}_i}{|\boldsymbol{v}_i|^2}e^{-\gamma \lambda_{i} t}\boldsymbol{v}_i\\
	&=\lim_{t\rightarrow \infty} \frac{X(0)\cdot \boldsymbol{v}_1}{|\boldsymbol{v}_1|^2}e^{-\gamma \lambda_{1} t}\boldsymbol{v}_1\\
	&= \frac{X(0)\cdot \boldsymbol{1}}{N} \cdot \boldsymbol{1}\\
	&= [m_c,m_c,\cdots,m_c]
	\end{align*}

	The result above implies that all data points eventually lie in the same position, which is their central point, as time $t$ approaches infinity. Moreover, it is obvious from the equation that the growth rate is exponential. Thus, with the evolution of our diffusion mechanism, the diameter $D(t)$ decreases exponentially to 0.
	
	For the more general case where there are $M$ connected components in the graph, the proof is identical by using the fact that the multiplicity of 0 eigenvalue of the Laplacian equals the number of connected components of graph G. For each connected component, its points will converge to a specific central, and the diameter of each subset
	\[\lim_{t\rightarrow \infty}\mathrm{diam}(\Gamma_m(t))=0,~\forall m \in [M]\]
	Thus the upper bound of diameters $D(t)$ decreases exponentially to 0.	Combining the results $L(t)$ is non-decreasing and $D(t)\rightarrow 0$ exponentially, we get \[\lim_{t\rightarrow \infty}\frac{L(t)}{D(t)}=\infty.\]
	the growth rate is exponential.
\end{proof}

%% file: data/appendix_experiments.tex
\section{Experiment Details and Results}

\subsection{Synthetic Data}
\label{section:experiment_synthetic}
\subsubsection{Dataset}
\textbf{XOR}\quad Uniformly collect 100 points each in four circles centered at (0,0), (0,2), (2,0), (2,2), respectively. Circles are with radius 0.75.\\
\textbf{Moon}\quad Uniformly collect 500 points each in two arcs of semi-circle: one is the upper arc of a circle centered at (0, 0) with radius 1, the other is the lower arc of a circle centered at (1, 0.5) also with radius 1. Points are added with a standard gaussian noise multiplied by 0.05.\\
\textbf{Circle}\quad Uniformly collect 500 points each in two circumference of circles: both are centered at (0, 0), one has radius 1 and the other has radius 2. Points are added with a standard gaussian noise multiplied by 0.05.\\
\textbf{Spiral}\quad Uniformly collect 500 points each in two spirals: both are parametrized by $r=a+b\theta$. One has $a=b=1$ and the other has $a=b=-1$. Points are added with a standard gaussian noise multiplied by 0.1.

\subsubsection{Network Structure}
\begin{align*}
	x&=x + \text{FC2 ( ReLU ( FC1 ( $x$ ) ) )}\\
	x&=\text{Diffusion}(x)  \qquad \text{for $r$ times}\\
	y&=\text{FC3 ( $x$ )}
\end{align*}
	
We use one residual block, i.e. $s=1$. All the fully connected layers are of size 2$\times$2 with bias (that is why we have totally 3$\times$2$\times$3=18 parameters). As for the diffusion layer, we use a fixed step size $\gamma$, and iterate for $r$ times.

\subsubsection{Parameters}
\begin{table}[hbtp]
    \caption{Parameters for synthetic data}
	\centering
	\begin{tabular}{lcccc}
		\toprule
		       & $n_{\mathrm{top}}$ & $\sigma$ & $\gamma$ & $r$ \\ \midrule
		XOR    & 20                 & 0.5      & 1.0      & /   \\
		Moon   & 25                 & 0.5      & 1.0      & 60  \\
		Circle & 50                 & 0.5      & 1.0      & 200 \\
		Spiral & 25                 & 0.5      & 1.0      & 900 \\ \bottomrule
	\end{tabular}
\end{table}

For the classification tasks, our optimizer is SGD with lr$=1.0$, momentum$=0.9$ and weight\_decay$=5e-4$. For spiral dataset, we adjust lr$=0.8$.

\subsubsection{Additional Results}
\label{section:additional_results}
We provide the figures describing the evolution of features with or without diffusion in residual network on the other two synthetic datasets in Figure~\ref{fig:two_moon} and Figure~\ref{fig:two_spiral}.
\begin{figure}[hbtp]
	\centering
	\subfloat[raw]{\raisebox{1.6ex}{\includegraphics[width=0.45\linewidth]{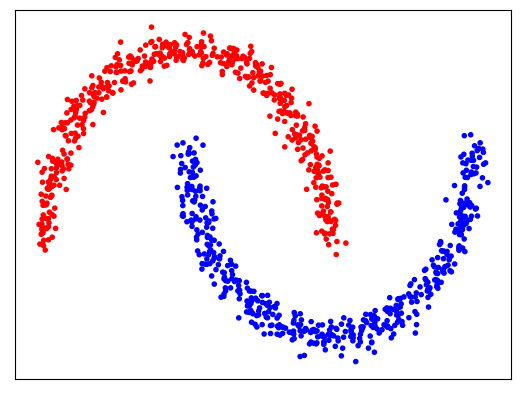}}}
	\hfill
	\subfloat[accuracy]{\includegraphics[width=0.5\linewidth]{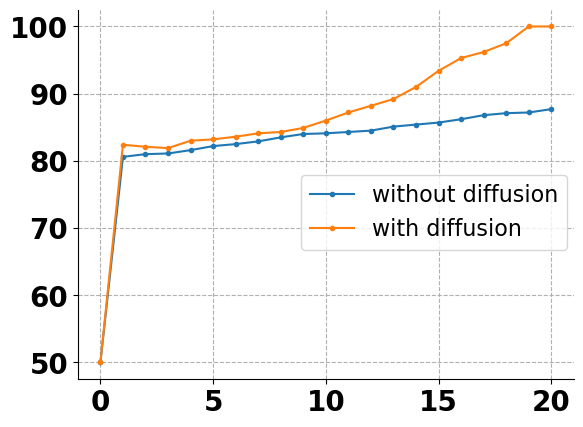}}
	
	\subfloat[w/o, epoch=0]{\includegraphics[width=0.33\linewidth]{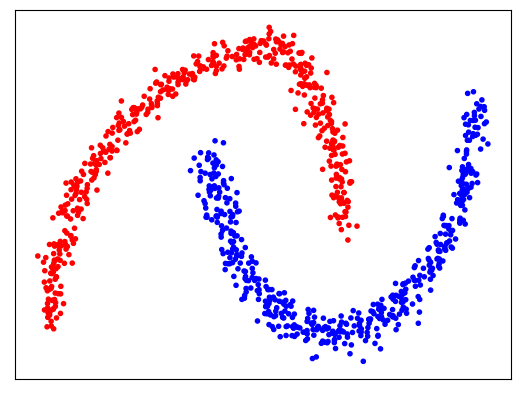}}
	\hfill
	\subfloat[w/o, epoch=10]{\includegraphics[width=0.33\linewidth]{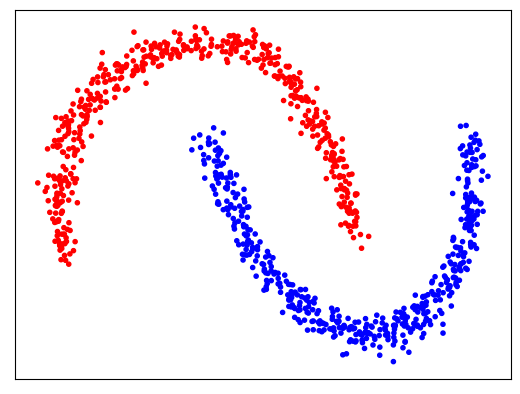}}
	\hfill
	\subfloat[w/o, epoch=20]{\includegraphics[width=0.33\linewidth]{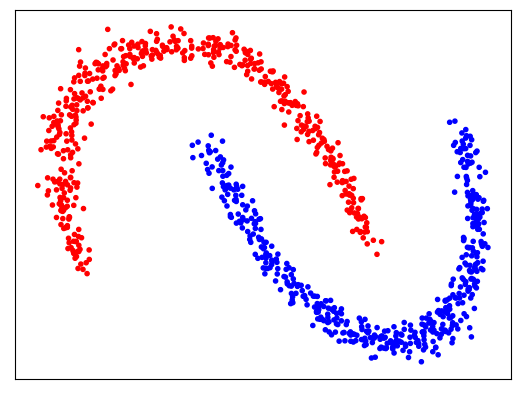}}
	
	\subfloat[w, epoch=0]{\includegraphics[width=0.33\linewidth]{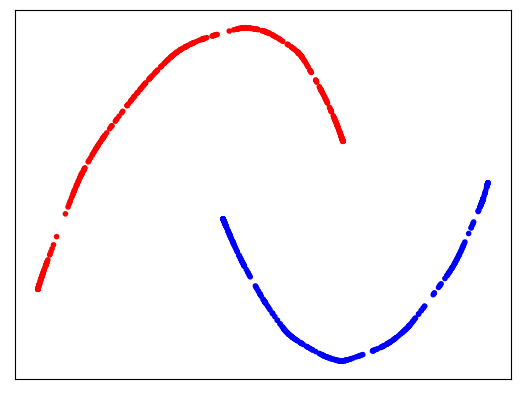}}
	\hfill
	\subfloat[w, epoch=10]{\includegraphics[width=0.33\linewidth]{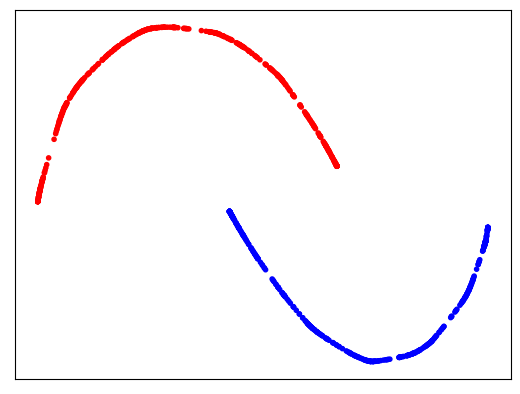}}
	\hfill
	\subfloat[w, epoch=20]{\includegraphics[width=0.33\linewidth]{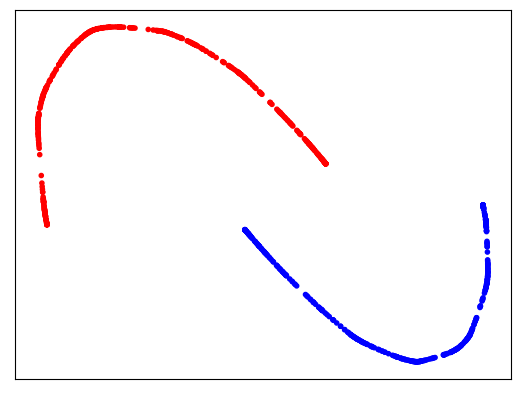}}
	\caption{ResNet and DiffResNet on moon dataset, figures are arranged similar to Fig.\ref{fig:two_circle}}
	\label{fig:two_moon}
\end{figure}

\begin{figure}[hbtp]
	\centering
	\subfloat[raw]{\raisebox{1.6ex}{\includegraphics[width=0.45\linewidth]{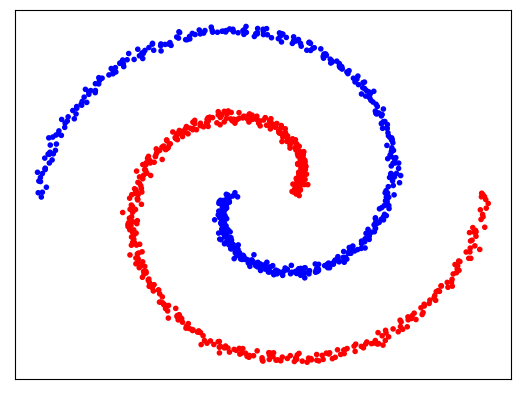}}}
	\hfill
	\subfloat[accuracy]{\includegraphics[width=0.5\linewidth]{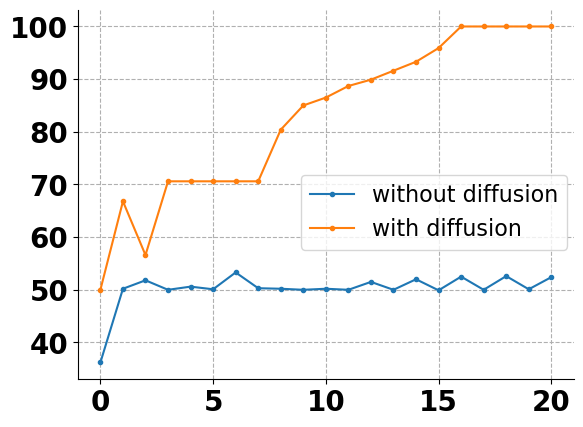}}
	
	\subfloat[w/o, epoch=0]{\includegraphics[width=0.33\linewidth]{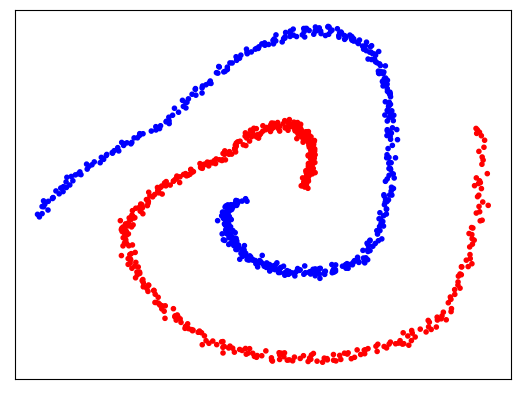}}
	\hfill
	\subfloat[w/o, epoch=10]{\includegraphics[width=0.33\linewidth]{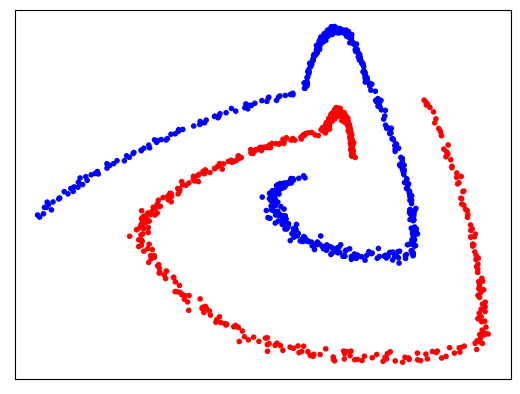}}
	\hfill
	\subfloat[w/o, epoch=20]{\includegraphics[width=0.33\linewidth]{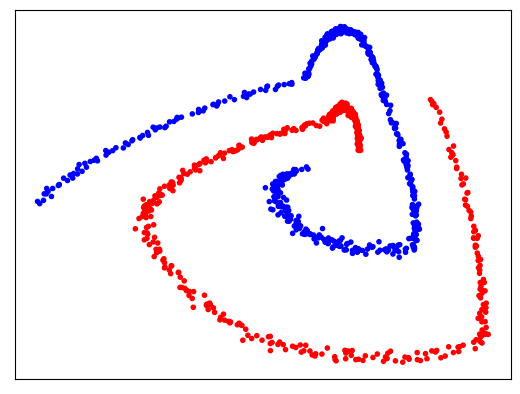}}
	
	\subfloat[w, epoch=0]{\includegraphics[width=0.33\linewidth]{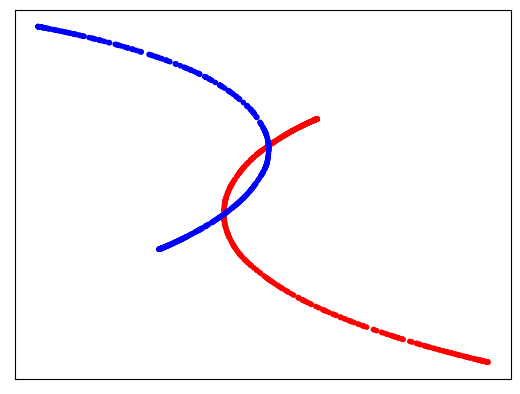}}
	\hfill
	\subfloat[w, epoch=10]{\includegraphics[width=0.33\linewidth]{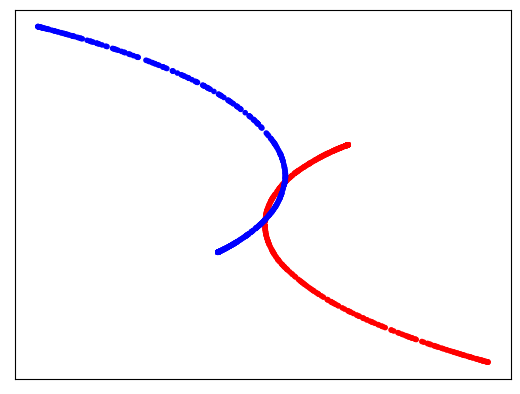}}
	\hfill
	\subfloat[w, epoch=20]{\includegraphics[width=0.33\linewidth]{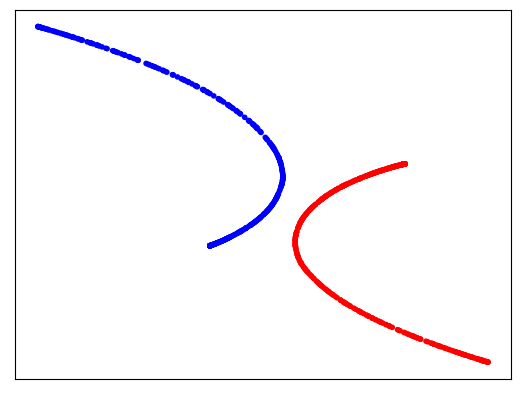}}
	\caption{ResNet and Diff-ResNet on spiral dataset, figures are arranged similar to Fig.\ref{fig:two_circle}}
	\label{fig:two_spiral}
\end{figure}

\subsection{Graph Learning}
\label{section:experiment_graph}
\subsubsection{Dataset}
\label{section:experiment_graph_dataset}
Here we give the statistics of each dataset. For each randomly chosen split, we pick 20 labeled points for training, and 30 points for validation in each class. All of the rest points are used as the test set. For all datasets, we treat the graph as undirected and only consider the largest connected component.
\begin{table}[hbtp]
	\caption{Graph Dataset Statistics.}
	\centering
	\begin{tabular}{lccccc}
		\toprule
		Dataset & Node & Edge & Class & Feature Dim & Label Rate \\ \midrule
		Cora    & 2485  & 5069  & 7       & 1433        & 0.057     \\
		Citeseer    & 2120  & 3679  & 6       & 3703        & 0.056     \\
		Pubmed    & 19717  & 44324  & 3       & 500        & 0.003     \\ \bottomrule
	\end{tabular}
\end{table}

\subsubsection{Preprocessing}
We follow the normalization technique in GCN~\cite{kipf2017semi}: the adjacent matrix is first added with a self-loop, and then symmetrically normalized. The feature vectors are row normalized.

\subsubsection{Network Structure}
Since we observe severe overfitting problem in graph learning, we delete FC2 to reduce the number of parameters, and apply dropout on the feature vectors after each round of diffusion.The network structure is: 
\begin{align*}
	x&=x + \text{ReLU ( FC1 ( $x$ ) )}\\
	x&=\text{Dropout ( Diffusion ( $x$ ) )}  \qquad \text{for $r$ times}\\
	y&=\text{FC3 ( $x$ )}
\end{align*}

The fully connected layers have input and output dimension the same as feature dimension. The new structure introduces a new parameter compared to toy examples: the dropout rate. But parameter $n_{\mathrm{top}}$ and $\sigma$ is unnecessary in graph learning.

\subsubsection{Parameters}
Parameters are chosen based on the accuracy on the validation set.
\begin{table}[hbtp]
    \caption{Parameters for graph learning}
	\centering
	\begin{tabular}{lccc}
		\toprule
		         & $\gamma$ & $r$  & dropout rate \\ \midrule
		Cora     & 20       & 0.25 & 0.25       \\
		Citeseer & 20       & 0.2  & 0.35        \\
		Pubmed   & 10       & 0.4  & 0.3        \\ \bottomrule
	\end{tabular}
\end{table}

\subsection{Few-shot Learning}
\label{section:experiment_fewshot}
\subsubsection{Preprocessing}
Following previous works\cite{wang2019simpleshot, ziko2020laplacian, liu2019prototype}, three additional feature transformation skills are used to enhance the performance.

\noindent\textbf{(1) Centering and Normalization} $$x=x-\bar{x} \quad \text{then} \quad x=\frac{x}{\|x\|_2},~\forall x\in \mathbb{X}_s\cup\mathbb{X}_q$$
$\bar{x}$ is the base class average. 

\noindent\textbf{(2) Cross-Domain Shift} \(x=x+\Delta,~\forall x\in \mathbb{X}_q\), where $$\Delta=\frac{1}{|\mathbb{X}_s|}\sum_{\mathbb{X}_s}x-\frac{1}{|\mathbb{X}_q|}\sum_{\mathbb{X}_q}x$$
is the difference between the mean of features within the support set and the mean of features within the query set. 

we apply the first two transformations to all extracted features.

\noindent\textbf{(3) Prototype Rectification}
$$\tilde{m_c}=\frac{1}{|\mathbb{X}_s^c|+|\mathbb{X}_q^c|}\sum_{x\in\{\mathbb{X}_s^c,\mathbb{X}_q^c\}}\frac{\exp(\cos(x,m_c))}{\sum_{x\in\{\mathbb{X}_s^c,\mathbb{X}_q^c\}}\exp(\cos(x,m_c))}x$$
Here $m_c=\frac{1}{|\mathbb{X}_s^c|}\sum_{\mathbb{X}_s^c}x$ is the mean of features within the support set of class $c$. $\mathbb{X}_s^c$ is the support set with label $c$. $\mathbb{X}_q^c$ is a pre-classified set based on nearest neighbors.

Prototype rectification is only applicable to classification methods that are based on prototype, and cannot be directly applied to our Cross-Entropy loss. Nonetheless, we observe in the ablation study~\ref{table:ablation} that in 5-shot tasks, merely nearest prototype classification can already achieve very competitive results, indicating the effectiveness of prototype in 5-shot tasks. So we mimic the first loss term in~\cite{ziko2020laplacian} and propose the prototypical loss below.
$$\text{Prototypical Loss}=\sum_{x_i \in \mathbb{X}_q} \sum_{c=1}^C f(x_i)_c d(x_i-\tilde{m_c}) $$

The final loss is a weighted sum of Cross-Entropy Loss and Prototypical Loss, with another parameter $\alpha$ before Prototypical Loss. 

\subsubsection{Network Structure}
\label{section:experiment_fewshot_network}
The structure is identical to that in the toy examples, just changing the input and output dimension of each FC layer from 2 to $M$, where $M$ is the dimension of embedded features.

\subsubsection{Parameters}
In the few-shot setting, the meaning of $\sigma$ is slightly different: weight is now calculated by $\tilde w_{ij}= \exp(-\|x_i-x_j\|_2^2/\sigma(x_i)^2)$, where $\sigma(x_i) = k$ means $\sigma$ is chosen to be the $k$-th closest distance from a specific point $x_i$ , so it varies with points.

In ablation study, $n_{\mathrm{top}}=8$, $\sigma=4$. The diffusion step size $\gamma$ is fixed to be 0.5 for all tasks. The diffusion step number $r=10$ for 1-shot learning, $r=5$ for 5-shot learning. $\lambda=0.5$, $\mu=0.01$.

In experiments with additional tricks, we also choose $n_{\mathrm{top}}=8$, $\sigma=4$. The diffusion step size $\gamma$ is fixed to be 0.5 for all tasks. The diffusion step number $r$ varies with tasks: for 1-shot learning, $r=6$ for \textit{mini}ImageNet and \textit{tiered}ImageNet, $r=7$ for CUB; for 5-shot learning, $r=3$ for all datasets and backbones. In addition, the weight before Prototypical Loss is $\alpha=0$ for 1-shot tasks, and $\alpha=0.5$ for 5-shot tasks.

The optimizer is SGD with initial learning rate$=0.1$, momentum$=0.9$ and weight\_decay$=1e$-4. We train $T=100$ epochs. We use a multi-step scheduler, which decays the learning rate by 0.1 at $0.5T$ and $0.75T$.

%% file: main.bbl
\begin{thebibliography}{100}
\providecommand{\url}[1]{#1}
\csname url@samestyle\endcsname
\providecommand{\newblock}{\relax}
\providecommand{\bibinfo}[2]{#2}
\providecommand{\BIBentrySTDinterwordspacing}{\spaceskip=0pt\relax}
\providecommand{\BIBentryALTinterwordstretchfactor}{4}
\providecommand{\BIBentryALTinterwordspacing}{\spaceskip=\fontdimen2\font plus
\BIBentryALTinterwordstretchfactor\fontdimen3\font minus
  \fontdimen4\font\relax}
\providecommand{\BIBforeignlanguage}[2]{{%
\expandafter\ifx\csname l@#1\endcsname\relax
\typeout{** WARNING: IEEEtran.bst: No hyphenation pattern has been}%
\typeout{** loaded for the language `#1'. Using the pattern for}%
\typeout{** the default language instead.}%
\else
\language=\csname l@#1\endcsname
\fi
#2}}
\providecommand{\BIBdecl}{\relax}
\BIBdecl

\bibitem{he2016deep}
K.~He, X.~Zhang, S.~Ren, and J.~Sun, ``Deep residual learning for image
  recognition,'' in \emph{Proceedings of the IEEE conference on computer vision
  and pattern recognition}, 2016, pp. 770--778.

\bibitem{Weinan2017A}
E.~Weinan, ``A proposal on machine learning via dynamical systems,''
  \emph{Communications in Mathematics \& Statistics}, vol.~5, no.~1, pp. 1--11,
  2017.

\bibitem{chen2018neural}
T.~Q. Chen, Y.~Rubanova, J.~Bettencourt, and D.~K. Duvenaud, ``Neural ordinary
  differential equations,'' in \emph{Advances in neural information processing
  systems}, 2018, pp. 6571--6583.

\bibitem{haber2018stable}
E.~Haber and L.~Ruthotto, ``Stable architectures for deep neural networks,''
  \emph{Inverse Problems}, vol.~34, no.~1, p. 014004, 2018.

\bibitem{zhang2017polynet}
X.~Zhang, Z.~Li, C.~Change~Loy, and D.~Lin, ``Polynet: A pursuit of structural
  diversity in very deep networks,'' in \emph{Proceedings of the IEEE
  Conference on Computer Vision and Pattern Recognition}, 2017, pp. 718--726.

\bibitem{larsson2016fractalnet}
G.~Larsson, M.~Maire, and G.~Shakhnarovich, ``Fractalnet: Ultra-deep neural
  networks without residuals,'' \emph{arXiv preprint arXiv:1605.07648}, 2016.

\bibitem{lu2018beyond}
Y.~Lu, A.~Zhong, Q.~Li, and B.~Dong, ``Beyond finite layer neural networks:
  Bridging deep architectures and numerical differential equations,'' in
  \emph{International Conference on Machine Learning}.\hskip 1em plus 0.5em
  minus 0.4em\relax PMLR, 2018, pp. 3276--3285.

\bibitem{Gastaldi17ShakeShake}
X.~Gastaldi, ``Shake-shake regularization,'' \emph{arXiv preprint
  arXiv:1705.07485}, 2017.

\bibitem{huang2016deep}
G.~Huang, Y.~Sun, Z.~Liu, D.~Sedra, and K.~Q. Weinberger, ``Deep networks with
  stochastic depth,'' in \emph{European conference on computer vision}.\hskip
  1em plus 0.5em minus 0.4em\relax Springer, 2016, pp. 646--661.

\bibitem{yang2020interpolation}
Z.~Yang, Y.~Liu, C.~Bao, and Z.~Shi, ``Interpolation between residual and
  non-residual networks,'' in \emph{International Conference on Machine
  Learning}.\hskip 1em plus 0.5em minus 0.4em\relax PMLR, 2020, pp.
  10\,736--10\,745.

\bibitem{lagaris1998artificial}
I.~E. Lagaris, A.~Likas, and D.~I. Fotiadis, ``Artificial neural networks for
  solving ordinary and partial differential equations,'' \emph{IEEE
  transactions on neural networks}, vol.~9, no.~5, pp. 987--1000, 1998.

\bibitem{dissanayake1994neural}
M.~Dissanayake and N.~Phan-Thien, ``Neural-network-based approximations for
  solving partial differential equations,'' \emph{communications in Numerical
  Methods in Engineering}, vol.~10, no.~3, pp. 195--201, 1994.

\bibitem{lagaris2000neural}
I.~E. Lagaris, A.~C. Likas, and D.~G. Papageorgiou, ``Neural-network methods
  for boundary value problems with irregular boundaries,'' \emph{IEEE
  Transactions on Neural Networks}, vol.~11, no.~5, pp. 1041--1049, 2000.

\bibitem{mcfall2009artificial}
K.~S. McFall and J.~R. Mahan, ``Artificial neural network method for solution
  of boundary value problems with exact satisfaction of arbitrary boundary
  conditions,'' \emph{IEEE Transactions on Neural Networks}, vol.~20, no.~8,
  pp. 1221--1233, 2009.

\bibitem{baymani2010artificial}
M.~Baymani, A.~Kerayechian, and S.~Effati, ``Artificial neural networks
  approach for solving stokes problem,'' \emph{Applied Mathematics}, vol.~1,
  no.~4, p. 288, 2010.

\bibitem{han2018solving}
J.~Han, A.~Jentzen, and E.~Weinan, ``Solving high-dimensional partial
  differential equations using deep learning,'' \emph{Proceedings of the
  National Academy of Sciences}, vol. 115, no.~34, pp. 8505--8510, 2018.

\bibitem{zhu2009introduction}
X.~Zhu and A.~B. Goldberg, ``Introduction to semi-supervised learning,''
  \emph{Synthesis lectures on artificial intelligence and machine learning},
  vol.~3, no.~1, pp. 1--130, 2009.

\bibitem{chapelle2009semi}
O.~Chapelle, B.~Scholkopf, and A.~Zien, ``Semi-supervised learning,''
  \emph{IEEE Transactions on Neural Networks}, vol.~20, no.~3, pp. 542--542,
  2009.

\bibitem{fei2006one}
L.~Fei-Fei, R.~Fergus, and P.~Perona, ``One-shot learning of object
  categories,'' \emph{IEEE transactions on pattern analysis and machine
  intelligence}, vol.~28, no.~4, pp. 594--611, 2006.

\bibitem{vinyals2016matching}
O.~Vinyals, C.~Blundell, T.~Lillicrap, D.~Wierstra \emph{et~al.}, ``Matching
  networks for one shot learning,'' \emph{Advances in Neural Information
  Processing Systems}, vol.~29, pp. 3630--3638, 2016.

\bibitem{wang2020generalizing}
Y.~Wang, Q.~Yao, J.~T. Kwok, and L.~M. Ni, ``Generalizing from a few examples:
  A survey on few-shot learning,'' \emph{ACM Computing Surveys (CSUR)},
  vol.~53, no.~3, pp. 1--34, 2020.

\bibitem{oliver2018realistic}
A.~Oliver, A.~Odena, C.~A. Raffel, E.~D. Cubuk, and I.~Goodfellow, ``Realistic
  evaluation of deep semi-supervised learning algorithms,'' \emph{Advances in
  neural information processing systems}, vol.~31, 2018.

\bibitem{bachman2014learning}
P.~Bachman, O.~Alsharif, and D.~Precup, ``Learning with pseudo-ensembles,''
  \emph{Advances in neural information processing systems}, vol.~27, 2014.

\bibitem{laine2016temporal}
S.~Laine and T.~Aila, ``Temporal ensembling for semi-supervised learning,''
  \emph{arXiv preprint arXiv:1610.02242}, 2016.

\bibitem{tarvainen2017mean}
A.~Tarvainen and H.~Valpola, ``Mean teachers are better role models:
  Weight-averaged consistency targets improve semi-supervised deep learning
  results,'' \emph{Advances in neural information processing systems}, vol.~30,
  2017.

\bibitem{miyato2018virtual}
T.~Miyato, S.-i. Maeda, M.~Koyama, and S.~Ishii, ``Virtual adversarial
  training: a regularization method for supervised and semi-supervised
  learning,'' \emph{IEEE transactions on pattern analysis and machine
  intelligence}, vol.~41, no.~8, pp. 1979--1993, 2018.

\bibitem{grandvalet2004semi}
Y.~Grandvalet and Y.~Bengio, ``Semi-supervised learning by entropy
  minimization,'' \emph{Advances in neural information processing systems},
  vol.~17, 2004.

\bibitem{lee2013pseudo}
D.-H. Lee \emph{et~al.}, ``Pseudo-label: The simple and efficient
  semi-supervised learning method for deep neural networks,'' in \emph{Workshop
  on challenges in representation learning, ICML}, vol.~3, no.~2, 2013, p. 896.

\bibitem{berthelot2019mixmatch}
D.~Berthelot, N.~Carlini, I.~Goodfellow, N.~Papernot, A.~Oliver, and C.~A.
  Raffel, ``Mixmatch: A holistic approach to semi-supervised learning,''
  \emph{Advances in Neural Information Processing Systems}, vol.~32, 2019.

\bibitem{sohn2020fixmatch}
K.~Sohn, D.~Berthelot, N.~Carlini, Z.~Zhang, H.~Zhang, C.~A. Raffel, E.~D.
  Cubuk, A.~Kurakin, and C.-L. Li, ``Fixmatch: Simplifying semi-supervised
  learning with consistency and confidence,'' \emph{Advances in Neural
  Information Processing Systems}, vol.~33, pp. 596--608, 2020.

\bibitem{morton2019numerical}
K.~W. Morton, \emph{Numerical solution of convection-diffusion problems}.\hskip
  1em plus 0.5em minus 0.4em\relax CRC Press, 2019.

\bibitem{svoboda2000convective}
Z.~Svoboda, ``The convective-diffusion equation and its use in building
  physics,'' \emph{International journal on architectural science}, vol.~1,
  no.~2, pp. 68--79, 2000.

\bibitem{markowich1989system}
P.~A. Markowich and P.~Szmolyan, ``A system of convection—diffusion equations
  with small diffusion coefficient arising in semiconductor physics,''
  \emph{Journal of Differential Equations}, vol.~81, no.~2, pp. 234--254, 1989.

\bibitem{zhu2003semi}
X.~Zhu, Z.~Ghahramani, and J.~D. Lafferty, ``Semi-supervised learning using
  gaussian fields and harmonic functions,'' in \emph{Proceedings of the 20th
  International conference on Machine learning (ICML-03)}, 2003, pp. 912--919.

\bibitem{nadler2009semi}
B.~Nadler, N.~Srebro, and X.~Zhou, ``Semi-supervised learning with the graph
  laplacian: The limit of infinite unlabelled data,'' \emph{Advances in neural
  information processing systems}, vol.~22, pp. 1330--1338, 2009.

\bibitem{shi2017weighted}
Z.~Shi, S.~Osher, and W.~Zhu, ``Weighted nonlocal laplacian on interpolation
  from sparse data,'' \emph{Journal of Scientific Computing}, vol.~73, no.~2,
  pp. 1164--1177, 2017.

\bibitem{weston2012deep}
J.~Weston, F.~Ratle, H.~Mobahi, and R.~Collobert, ``Deep learning via
  semi-supervised embedding,'' in \emph{Neural networks: Tricks of the
  trade}.\hskip 1em plus 0.5em minus 0.4em\relax Springer, 2012, pp. 639--655.

\bibitem{yang2016revisiting}
Z.~Yang, W.~Cohen, and R.~Salakhudinov, ``Revisiting semi-supervised learning
  with graph embeddings,'' in \emph{International conference on machine
  learning}.\hskip 1em plus 0.5em minus 0.4em\relax PMLR, 2016, pp. 40--48.

\bibitem{li2018learning}
Y.~Li and Y.~Liang, ``Learning overparameterized neural networks via stochastic
  gradient descent on structured data,'' in \emph{Advances in Neural
  Information Processing Systems}, 2018, pp. 8157--8166.

\bibitem{klicpera2019diffusion}
J.~Klicpera, S.~Wei{\ss}enberger, and S.~G{\"u}nnemann, ``Diffusion improves
  graph learning,'' in \emph{Conference on Neural Information Processing
  Systems (NeurIPS)}, 2019.

\bibitem{atwood2016diffusion}
J.~Atwood and D.~Towsley, ``Diffusion-convolutional neural networks,'' in
  \emph{Advances in neural information processing systems}, 2016, pp.
  1993--2001.

\bibitem{coifman2005geometric}
R.~R. Coifman, S.~Lafon, A.~B. Lee, M.~Maggioni, B.~Nadler, F.~Warner, and
  S.~W. Zucker, ``Geometric diffusions as a tool for harmonic analysis and
  structure definition of data: Diffusion maps,'' \emph{Proceedings of the
  national academy of sciences}, vol. 102, no.~21, pp. 7426--7431, 2005.

\bibitem{belkin2003laplacian}
M.~Belkin and P.~Niyogi, ``Laplacian eigenmaps for dimensionality reduction and
  data representation,'' \emph{Neural computation}, vol.~15, no.~6, pp.
  1373--1396, 2003.

\bibitem{kushnir2020diffusion}
D.~Kushnir and L.~Venturi, ``Diffusion-based deep active learning,''
  \emph{arXiv preprint arXiv:2003.10339}, 2020.

\bibitem{jiang2018difnet}
P.~Jiang, F.~Gu, Y.~Wang, C.~Tu, and B.~Chen, ``Difnet: Semantic segmentation
  by diffusion networks,'' \emph{arXiv preprint arXiv:1805.08015}, 2018.

\bibitem{wang2018resnets}
B.~Wang, B.~Yuan, Z.~Shi, and S.~Osher, ``Resnets ensemble via the feynman-kac
  formalism to improve natural and robust accuracies,'' \emph{arXiv: Learning},
  2018.

\bibitem{chamberlain2021grand}
B.~Chamberlain, J.~Rowbottom, M.~I. Gorinova, M.~Bronstein, S.~Webb, and
  E.~Rossi, ``Grand: Graph neural diffusion,'' in \emph{International
  Conference on Machine Learning}.\hskip 1em plus 0.5em minus 0.4em\relax PMLR,
  2021, pp. 1407--1418.

\bibitem{rodriguez2020embedding}
P.~Rodr{\'\i}guez, I.~Laradji, A.~Drouin, and A.~Lacoste, ``Embedding
  propagation: Smoother manifold for few-shot classification,'' in
  \emph{Computer Vision--ECCV 2020: 16th European Conference, Glasgow, UK,
  August 23--28, 2020, Proceedings, Part XXVI 16}.\hskip 1em plus 0.5em minus
  0.4em\relax Springer, 2020, pp. 121--138.

\bibitem{li2017maximum}
Q.~Li, L.~Chen, C.~Tai, and E.~Weinan, ``Maximum principle based algorithms for
  deep learning,'' \emph{Journal of Machine Learning Research}, vol.~18, no.
  165, pp. 1--29, 2017.

\bibitem{li2018an}
Q.~Li and S.~Hao, ``An optimal control approach to deep learning and
  applications to discrete-weight neural networks,'' \emph{arXiv: Learning},
  2018.

\bibitem{zhou2004learning}
D.~Zhou, O.~Bousquet, T.~N. Lal, J.~Weston, and B.~Sch{\"o}lkopf, ``Learning
  with local and global consistency,'' in \emph{Advances in neural information
  processing systems}, 2004, pp. 321--328.

\bibitem{belkin2004semi}
M.~Belkin and P.~Niyogi, ``Semi-supervised learning on riemannian manifolds,''
  \emph{Machine learning}, vol.~56, no.~1, pp. 209--239, 2004.

\bibitem{belkin2006manifold}
M.~Belkin, P.~Niyogi, and V.~Sindhwani, ``Manifold regularization: A geometric
  framework for learning from labeled and unlabeled examples.'' \emph{Journal
  of machine learning research}, vol.~7, no.~11, 2006.

\bibitem{ando2007learning}
R.~K. Ando and T.~Zhang, ``Learning on graph with laplacian regularization,''
  in \emph{Advances in neural information processing systems}, 2007, pp.
  25--32.

\bibitem{smola2003kernels}
A.~J. Smola and R.~Kondor, ``Kernels and regularization on graphs,'' in
  \emph{Learning theory and kernel machines}.\hskip 1em plus 0.5em minus
  0.4em\relax Springer, 2003, pp. 144--158.

\bibitem{chen2020simple}
T.~Chen, S.~Kornblith, M.~Norouzi, and G.~Hinton, ``A simple framework for
  contrastive learning of visual representations,'' in \emph{International
  conference on machine learning}.\hskip 1em plus 0.5em minus 0.4em\relax PMLR,
  2020, pp. 1597--1607.

\bibitem{caron2020unsupervised}
M.~Caron, I.~Misra, J.~Mairal, P.~Goyal, P.~Bojanowski, and A.~Joulin,
  ``Unsupervised learning of visual features by contrasting cluster
  assignments,'' \emph{Advances in neural information processing systems},
  vol.~33, pp. 9912--9924, 2020.

\bibitem{caron2021emerging}
M.~Caron, H.~Touvron, I.~Misra, H.~J{\'e}gou, J.~Mairal, P.~Bojanowski, and
  A.~Joulin, ``Emerging properties in self-supervised vision transformers,'' in
  \emph{Proceedings of the IEEE/CVF international conference on computer
  vision}, 2021, pp. 9650--9660.

\bibitem{dosovitskiy2021image}
A.~Dosovitskiy, L.~Beyer, A.~Kolesnikov, D.~Weissenborn, X.~Zhai,
  T.~Unterthiner, M.~Dehghani, M.~Minderer, G.~Heigold, S.~Gelly \emph{et~al.},
  ``An image is worth 16x16 words: Transformers for image recognition at
  scale,'' \emph{International Conference on Learning Representations}, 2021.

\bibitem{satorras2018few}
V.~G. Satorras and J.~B. Estrach, ``Few-shot learning with graph neural
  networks,'' in \emph{International conference on learning representations},
  2018.

\bibitem{gidaris2019generating}
S.~Gidaris and N.~Komodakis, ``Generating classification weights with gnn
  denoising autoencoders for few-shot learning,'' in \emph{Proceedings of the
  IEEE/CVF conference on computer vision and pattern recognition}, 2019, pp.
  21--30.

\bibitem{zhu2022ease}
H.~Zhu and P.~Koniusz, ``Ease: Unsupervised discriminant subspace learning for
  transductive few-shot learning,'' in \emph{Proceedings of the IEEE/CVF
  Conference on Computer Vision and Pattern Recognition}, 2022, pp. 9078--9088.

\bibitem{miller2000learning}
E.~G. Miller, N.~E. Matsakis, and P.~A. Viola, ``Learning from one example
  through shared densities on transforms,'' in \emph{Proceedings IEEE
  Conference on Computer Vision and Pattern Recognition. CVPR 2000 (Cat. No.
  PR00662)}, vol.~1.\hskip 1em plus 0.5em minus 0.4em\relax IEEE, 2000, pp.
  464--471.

\bibitem{hochreiter1997long}
S.~Hochreiter and J.~Schmidhuber, ``Long short-term memory,'' \emph{Neural
  computation}, vol.~9, no.~8, pp. 1735--1780, 1997.

\bibitem{snell2017prototypical}
J.~Snell, K.~Swersky, and R.~S. Zemel, ``Prototypical networks for few-shot
  learning,'' \emph{arXiv preprint arXiv:1703.05175}, 2017.

\bibitem{sung2018learning}
F.~Sung, Y.~Yang, L.~Zhang, T.~Xiang, P.~H. Torr, and T.~M. Hospedales,
  ``Learning to compare: Relation network for few-shot learning,'' in
  \emph{Proceedings of the IEEE conference on computer vision and pattern
  recognition}, 2018, pp. 1199--1208.

\bibitem{oreshkin2018tadam}
B.~N. Oreshkin, P.~Rodriguez, and A.~Lacoste, ``Tadam: Task dependent adaptive
  metric for improved few-shot learning,'' \emph{arXiv preprint
  arXiv:1805.10123}, 2018.

\bibitem{wang2019simpleshot}
Y.~Wang, W.-L. Chao, K.~Q. Weinberger, and L.~van~der Maaten, ``Simpleshot:
  Revisiting nearest-neighbor classification for few-shot learning,''
  \emph{arXiv preprint arXiv:1911.04623}, 2019.

\bibitem{ziko2020laplacian}
I.~Ziko, J.~Dolz, E.~Granger, and I.~B. Ayed, ``Laplacian regularized few-shot
  learning,'' in \emph{International Conference on Machine Learning}.\hskip 1em
  plus 0.5em minus 0.4em\relax PMLR, 2020, pp. 11\,660--11\,670.

\bibitem{liu2019prototype}
J.~Liu, L.~Song, and Y.~Qin, ``Prototype rectification for few-shot learning,''
  \emph{arXiv preprint arXiv:1911.10713}, 2019.

\bibitem{yang2021free}
S.~Yang, L.~Liu, and M.~Xu, ``Free lunch for few-shot learning: Distribution
  calibration,'' in \emph{International Conference on Learning Representations
  (ICLR)}, 2021.

\bibitem{hu2021leveraging}
Y.~Hu, V.~Gripon, and S.~Pateux, ``Leveraging the feature distribution in
  transfer-based few-shot learning,'' in \emph{Artificial Neural Networks and
  Machine Learning--ICANN 2021: 30th International Conference on Artificial
  Neural Networks, Bratislava, Slovakia, September 14--17, 2021, Proceedings,
  Part II 30}.\hskip 1em plus 0.5em minus 0.4em\relax Springer, 2021, pp.
  487--499.

\bibitem{tukey1977exploratory}
J.~W. Tukey \emph{et~al.}, \emph{Exploratory data analysis}.\hskip 1em plus
  0.5em minus 0.4em\relax Reading, MA, 1977, vol.~2.

\bibitem{huang2019few}
G.~Huang, H.~Larochelle, and S.~Lacoste-Julien, ``Are few-shot learning
  benchmarks too simple? solving them without task supervision at test-time,''
  \emph{arXiv preprint arXiv:1902.08605}, 2019.

\bibitem{geiser2009decomposition}
J.~Geiser, \emph{Decomposition methods for differential equations: theory and
  applications}.\hskip 1em plus 0.5em minus 0.4em\relax CRC Press, 2009.

\bibitem{balister2005connectivity}
P.~Balister, B.~Bollob{\'a}s, A.~Sarkar, and M.~Walters, ``Connectivity of
  random k-nearest-neighbour graphs,'' \emph{Advances in Applied Probability},
  vol.~37, no.~1, pp. 1--24, 2005.

\bibitem{shchur2018pitfalls}
O.~Shchur, M.~Mumme, A.~Bojchevski, and S.~G{\"u}nnemann, ``Pitfalls of graph
  neural network evaluation,'' \emph{Relational Representation Learning
  Workshop, NeurIPS 2018}, 2018.

\bibitem{kipf2017semi}
T.~N. Kipf and M.~Welling, ``Semi-supervised classification with graph
  convolutional networks,'' in \emph{International Conference on Learning
  Representations (ICLR)}, 2017.

\bibitem{hamilton2017inductive}
W.~L. Hamilton, R.~Ying, and J.~Leskovec, ``Inductive representation learning
  on large graphs,'' in \emph{NIPS}, 2017.

\bibitem{velickovic2018graph}
P.~Veli{\v{c}}kovi{\'{c}}, G.~Cucurull, A.~Casanova, A.~Romero, P.~Li{\`{o}},
  and Y.~Bengio, ``{Graph Attention Networks},'' \emph{International Conference
  on Learning Representations}, 2018.

\bibitem{xhonneux2020continuous}
L.-P. Xhonneux, M.~Qu, and J.~Tang, ``Continuous graph neural networks,'' in
  \emph{International Conference on Machine Learning}.\hskip 1em plus 0.5em
  minus 0.4em\relax PMLR, 2020, pp. 10\,432--10\,441.

\bibitem{poli2019graph}
M.~Poli, S.~Massaroli, J.~Park, A.~Yamashita, H.~Asama, and J.~Park, ``Graph
  neural ordinary differential equations,'' \emph{arXiv preprint
  arXiv:1911.07532}, 2019.

\bibitem{li2018deeper}
Q.~Li, Z.~Han, and X.-M. Wu, ``Deeper insights into graph convolutional
  networks for semi-supervised learning,'' in \emph{Thirty-Second AAAI
  conference on artificial intelligence}, 2018.

\bibitem{russakovsky2015imagenet}
O.~Russakovsky, J.~Deng, H.~Su, J.~Krause, S.~Satheesh, S.~Ma, Z.~Huang,
  A.~Karpathy, A.~Khosla, M.~Bernstein \emph{et~al.}, ``Imagenet large scale
  visual recognition challenge,'' \emph{International journal of computer
  vision}, vol. 115, no.~3, pp. 211--252, 2015.

\bibitem{wah2011caltech}
C.~Wah, S.~Branson, P.~Welinder, P.~Perona, and S.~Belongie, ``{The
  Caltech-UCSD Birds-200-2011 Dataset},'' California Institute of Technology,
  Tech. Rep. CNS-TR-2011-001, 2011.

\bibitem{Sachin2017}
S.~Ravi and H.~Larochelle, ``Optimization as a model for few-shot learning,''
  in \emph{In International Conference on Learning Representations (ICLR)},
  2017.

\bibitem{chen2019closerfewshot}
W.-Y. Chen, Y.-C. Liu, Z.~Kira, Y.-C. Wang, and J.-B. Huang, ``A closer look at
  few-shot classification,'' in \emph{International Conference on Learning
  Representations}, 2019.

\bibitem{Zagoruyko2016WRN}
S.~Zagoruyko and N.~Komodakis, ``Wide residual networks,'' in \emph{BMVC},
  2016.

\bibitem{van2008visualizing}
L.~Van~der Maaten and G.~Hinton, ``Visualizing data using t-sne.''
  \emph{Journal of machine learning research}, vol.~9, no.~11, 2008.

\bibitem{finn2017model}
C.~Finn, P.~Abbeel, and S.~Levine, ``Model-agnostic meta-learning for fast
  adaptation of deep networks,'' in \emph{International Conference on Machine
  Learning}.\hskip 1em plus 0.5em minus 0.4em\relax PMLR, 2017, pp. 1126--1135.

\bibitem{gidaris2018dynamic}
S.~Gidaris and N.~Komodakis, ``Dynamic few-shot visual learning without
  forgetting,'' in \emph{Proceedings of the IEEE Conference on Computer Vision
  and Pattern Recognition}, 2018, pp. 4367--4375.

\bibitem{mishra2018simple}
N.~Mishra, M.~Rohaninejad, X.~Chen, and P.~Abbeel, ``A simple neural attentive
  meta-learner,'' in \emph{International Conference on Learning
  Representations}, 2018.

\bibitem{dhillon2019baseline}
G.~S. Dhillon, P.~Chaudhari, A.~Ravichandran, and S.~Soatto, ``A baseline for
  few-shot image classification,'' in \emph{International Conference on
  Learning Representations}, 2019.

\bibitem{lee2019meta}
K.~Lee, S.~Maji, A.~Ravichandran, and S.~Soatto, ``Meta-learning with
  differentiable convex optimization,'' in \emph{Proceedings of the IEEE/CVF
  Conference on Computer Vision and Pattern Recognition}, 2019, pp.
  10\,657--10\,665.

\bibitem{liu2018learning}
Y.~Liu, J.~Lee, M.~Park, S.~Kim, E.~Yang, S.~J. Hwang, and Y.~Yang, ``Learning
  to propagate labels: Transductive propagation network for few-shot
  learning,'' \emph{arXiv preprint arXiv:1805.10002}, 2018.

\bibitem{qiao2019transductive}
L.~Qiao, Y.~Shi, J.~Li, Y.~Wang, T.~Huang, and Y.~Tian, ``Transductive
  episodic-wise adaptive metric for few-shot learning,'' in \emph{Proceedings
  of the IEEE/CVF International Conference on Computer Vision}, 2019, pp.
  3603--3612.

\bibitem{hou2019cross}
R.~Hou, H.~Chang, B.~Ma, S.~Shan, and X.~Chen, ``Cross attention network for
  few-shot classification,'' in \emph{NeurIPS}, 2019.

\bibitem{qiao2018few}
S.~Qiao, C.~Liu, W.~Shen, and A.~L. Yuille, ``Few-shot image recognition by
  predicting parameters from activations,'' in \emph{Proceedings of the IEEE
  Conference on Computer Vision and Pattern Recognition}, 2018, pp. 7229--7238.

\bibitem{rusu2018meta}
A.~A. Rusu, D.~Rao, J.~Sygnowski, O.~Vinyals, R.~Pascanu, S.~Osindero, and
  R.~Hadsell, ``Meta-learning with latent embedding optimization,'' in
  \emph{International Conference on Learning Representations}, 2018.

\bibitem{gidaris2019boosting}
S.~Gidaris, A.~Bursuc, N.~Komodakis, P.~P{\'e}rez, and M.~Cord, ``Boosting
  few-shot visual learning with self-supervision,'' in \emph{Proceedings of the
  IEEE/CVF International Conference on Computer Vision}, 2019, pp. 8059--8068.

\bibitem{ye2020fewshot}
H.-J. Ye, H.~Hu, D.-C. Zhan, and F.~Sha, ``Few-shot learning via embedding
  adaptation with set-to-set functions,'' in \emph{IEEE/CVF Conference on
  Computer Vision and Pattern Recognition (CVPR)}, 2020, pp. 8808--8817.

\bibitem{li2011concise}
S.~Li, ``Concise formulas for the area and volume of a hyperspherical cap,''
  \emph{Asian Journal of Mathematics and Statistics}, vol.~4, no.~1, pp.
  66--70, 2011.

\bibitem{chung1997spectral}
F.~R. Chung and F.~C. Graham, \emph{Spectral graph theory}.\hskip 1em plus
  0.5em minus 0.4em\relax American Mathematical Soc., 1997, no.~92.

\end{thebibliography}
